\documentclass[10pt,letterpaper]{article}

\input{common-header.sty}
\title{Max-Min Grouped Bandits}

\author{
    Zhenlin Wang,\textsuperscript{\rm 1}
    Jonathan Scarlett\textsuperscript{\rm 1,2}\\
}
\affiliations {
    \textsuperscript{\rm 1}School of Computing, National University of Singapore\\
    \textsuperscript{\rm 2}Department of Mathematics \& Institute of Data Science, National University of Singapore\\
    wang{\textunderscore}zhenlin@u.nus.edu,
    scarlett@comp.nus.edu.sg
}

\begin{document}
\maketitle
\begin{abstract}
	In this paper, we introduce a multi-armed bandit problem termed max-min grouped bandits, in which the arms are arranged in possibly-overlapping groups, and the goal is to find the group whose worst arm has the highest mean reward.  This problem is of interest in applications such as recommendation systems and resource allocation, and is also closely related to widely-studied robust optimization problems.  We present two algorithms based successive elimination and robust optimization, and derive upper bounds on the number of samples to guarantee finding a max-min optimal or near-optimal group, as well as an algorithm-independent lower bound.  We discuss the degree of tightness of our bounds in various cases of interest, and the difficulties in deriving uniformly tight bounds.
\end{abstract}

\section{Introduction}

Multi-armed bandit (MAB) algorithms are widely adopted in scenarios of decision-making under uncertainty \cite{Csa18}.  In theoretical MAB studies, two particularly common performance goals are {\em regret minimization} and {\em best-arm identification}, and this paper is more closely related to the latter.

The most basic form of best-arm identification seeks to identify the arm with the highest mean (e.g., see \cite{kaufmann2016complexity}).  Other variations are instead based on returning {\em multiple arms}, such as the $k$ believed to have the highest $k$ means \cite{kalyanakrishnan2012pac}, or the individual highest-mean arms within a pre-defined set of groups that may be non-overlapping \cite{gabillon2011MultiBandit} or overlapping \cite{scarlett2019overlapping}.

In this paper, we introduce a distinct problem setup in which we are again given a collection of (possibly overlapping) groups of arms, but the goal is to identify the group whose {\em worst arm} (in terms of the mean reward) is as high as possible.  To motivate this problem setup, we list two potential applications:
\begin{itemize}
    \item In recommendation systems, the groups may correspond to packages of items than can be offered or advertised together.  If the users are highly averse to poor items, then it is natural to model the likelihood of clicking/purchasing as being dictated by the worst item.
    \item In a resource allocation setting, suppose that we would like to choose the best group of computing machines (or other resources), but we require robustness because the slowest machine will be bottleneck when it comes to running jobs. Then, we would like to find the group whose worst-case machine is the best.
\end{itemize}
More generally, this problem captures the notion of a group {\em only being as strong as its weakest link}, and is closely related to widely-studied robust optimization problems (e.g., \cite{Ber10}).

Before describing our main contributions, we provide a more detailed overview of the most related existing works.

\subsection{Related Work}

The related work on multi-armed bandits and best-arm identification is extensive; we only provide a brief outline here with an emphasis on the most closely related works.

The standard best-arm identification problem was studied in \cite{audibert2010bestarmid,gabillon2012unified,jamieson2014best,kaufmann2016complexity,garivier2016optimal}, among others.  These works are commonly distinguished according to whether the time horizon is fixed (fixed-budget setting) or the target error probability is fixed (fixed-confidence setting), and the latter is more relevant to our work.  In particular, we will utilize anytime confidence bounds from \cite{jamieson2014best} in our upper bounds, as well as a fundamental change-of-measure technique from \cite{kaufmann2016complexity} in our lower bounds.

A notable {\em grouped} best-arm identification problem was studied in \cite{gabillon2011MultiBandit,bubeck2013multipleident}, where the arms are partitioned into disjoint groups, and the goal is to find the best arm in each group.  A generalization of this problem to the case of overlapping groups was provided in  \cite{scarlett2019overlapping}.  Another notable setting in which multiple arms are returned is that of subset selection, where one seeks to find a subset of $k$ arms attaining the highest mean rewards \cite{kalyanakrishnan2012pac,pmlr-v30-Kaufmann13,kaufmann2016complexity}.  In our understanding, all of these works are substantially different from the max-min grouped bandit problem that we consider.

Another setup of interest is the recently-proposed categorized bandit problem \cite{jedor2020categorized}.  This setting consists of disjoint groups with a partial ordering, and the knowledge of the group structure (but not their order) is given as prior information.  However, different from our setting, the goal is still to find the best overall arm (or more precisely, minimize the corresponding regret notion).  In addition, the results of \cite{jedor2020categorized} are based on the arm means satisfying certain partial ordering assumptions between the groups (e.g., all arms in a better group beat all arms in a worse group), whereas we consider general instances without such restrictions.  See also \cite{bouneffouf2019optimal,ban2021local,singh2020multiarmed} and the references therein for other MAB settings with a clustering structure.

Our setup can be viewed as a MAB counterpart to {\em robust optimization}, which has received considerable attention on continuous domains \cite{Ber10,chen2017robust,bogunovic2018adversarially}, as well as set-valued domains with submodular functions \cite{krause2008robust,orlin2018robust,Bog17}.  Robust maximization problems generically take the form $\max\limits_{x \in D_x} \min\limits_{c \in D_c} f(x,c)$, and in \refsec{sec:stable} we will explicitly connect our setup to the kernelized robust optimization setting studied in \cite{bogunovic2018adversarially} (see also \cite{yang2021robust} for a related follow-up work).  However, based on what is currently known, this connection will only provide relatively loose instance-independent bounds when applied to our setting, and the bounds derived in our work (both instance-dependent and instance-independent) will require a separate treatment.

\subsection{Contributions and Paper Structure}

The paper is outlined as follows:
\begin{itemize}[noitemsep,topsep=0pt,parsep=0pt,partopsep=0pt]
    \item In \refsec{sec:setup}, we formally introduce the max-min grouped bandit problem, and briefly discuss a naive approach.  
    \item In \refsec{sec:se}, we present an algorithm based on successive elimination, and derive an instance-dependent upper bound on time required to find the optimal group.  
    \item In \refsec{sec:stable}, we show that our setup can be cast under the framework of kernel-based robust optimization, and use this connection to adapt an algorithm from \cite{bogunovic2018adversarially}.  We additionally derive an instance-independent regret bound (i.e., a bound on the suboptimality of the declared group relative to the best).  
    \item In \refsec{sec:lower}, we return to considering  instance-dependent bounds, and complement our upper bound with an algorithm-independent lower bound.  
    \item In \refsec{sec:discussion}, we further discuss our bounds, including highlighting cases where they are tight vs.~loose, and the difficulty in deriving uniformly tight bounds.
    \item In \refsec{sec:exp}, we present numerical experiments investigating the relative performance of the algorithms considered.
    % \item Discussions and conclusions are presented in \refsec{sec:discussion} and \refsec{sec:conclusion} respectively, 
\end{itemize}

\section{Problem Setup} \label{sec:setup}
%  Each time the agent pulls an arm $a_t \in {\cal A}$, a reward is observed following Bernoulli distribution with means $\mu_t$.
We first describe the problem aspects that are the same as the regular MAB problem.  We consider a collection ${\cal A} = \{1,\dotsc,n\}$ of $n$ arms/actions.  In each round, indexed by $t$, the MAB algorithm selects an arm $j_t \in {\cal A}$, and observes the corresponding reward $X_{j_t, T_{j_t}(t)}$, where $T_{j}(t)$ is the number of pulls of arm $j$ up to time $t$.  We consider stochastic rewards, in which for each $j \in {\cal A}$, the random variables $\{X_{j,\tau}\}_{\tau \ge 1}$ are independently drawn from some unknown distribution with mean $\mu_j$.  We will consider algorithms that make use of the empirical mean, defined as $\hat{\mu}_{j, T_j(t)} = \frac{1}{T_j(t)}\sum_{s=1}^{T_j(t)}X_{j,s}$.

Different from the standard MAB setup, we assume that there is a known set of groups ${\cal G}$, where each group $G \in \cal G$ is a non-empty subset of ${\cal A}$. We allow overlaps between groups, i.e., a given arm may appear in multiple groups.  Without loss of generality, we assume that each arm is in at least one group.  We are interested in identifying the {\em max-min optimal group}, defined as follows:
\begin{equation} \label{def: G*}
    G^* = \argmax_{G \in {\cal G}} \min_{j \in G} \mu_j.
\end{equation}
To lighten notation, we define $j_{\mathrm{worst}}(G) = \argmin_{j \in G} \mu_j$ to be the arm in $G$ with the lowest mean; if this is non-unique, we simply take any one of them chosen arbitrarily.

After $T$ rounds (where $T$ may be fixed in advance or adaptively chosen based on the rewards), the algorithm outputs a recommendation $G^{(T)}$ representing its guess of the optimal group.  We consider two closely-related performance measures, namely, the error probability
\begin{equation} \label{def: P_e }
    P_e = \mathbb{P}[G^{(T)} \neq G^*],
\end{equation}
and the simple regret
\begin{equation} \label{def: simple_regret}
    r(G^{(T)}) = \mu_{j_{\rm worst}(G^*)} - \mu_{j_{\rm worst}(G^{(T)})}.
\end{equation}
Naturally, we would like $P_e$ and/or $r(G^{(T)})$ to be as small as possible, while also using as few arm pulls as possible.

\subsection{Assumptions} \label{sec:assump}

Throughout the paper, we will make use of several assumptions that are either standard in the literature, or simple variations thereof.  We start with the following.

\begin{assumption} \label{as:noise}
    We assume that the arm means are bounded in $[0,1]$,\footnote{Any finite interval can be shifted and scaled to this range.} and that the reward distributions are sub-Gaussian with parameter $R$, i.e., if $X_j$ is a random variable drawn from the $j$-th arm's distribution, then $\mathbb{E}[X_{j}] = \mu_j$ and $\mathbb{E}[e^{\lambda (X_{j} - \mu_{j})}] \leq \exp(\lambda^2 R^2/2)$ for all $\lambda \in \mathbb{R}$.
\end{assumption}

We will consider Gaussian and Bernoulli rewards as canonical examples of distributions satisfying \refas{as:noise}.

The next assumption serves as a natural counterpart to that of having a unique best arm in the standard best-arm identification problem, i.e., an {\em identifiability} assumption.

\begin{assumption}  \label{as:unique}
    There exists a unique group $G^*$ with the highest worst arm, i.e., 
    \begin{equation} 
        \min\limits_{j \in G^*} \mu_j > \max\limits_{G \in {\cal G}: G \neq G^*}\min\limits_{j \in G} \mu_j.
    \end{equation} 
\end{assumption}

With this assumption in mind, we now turn to defining {\em fundamental gaps} between the arm means.  These are also ubiquitous in instance-dependent studies of MAB problems, but are somewhat different here compared to other settings.

Recall that $j_{\mathrm{worst}}(G)$ is the worst arm in a group $G \in {\cal G}$. We define the difference between the worst arm of $G^*$ and the worst arm of group $G$ as $\Delta_G = \mu_{j_{\mathrm{worst}}(G^*)} - \mu_{j_{\mathrm{worst}}(G)}$.  Then, for each arm indexed by $j$, the following quantities will play a key role in our analysis:
\begin{itemize}
    \item $\Delta_j' := \min\limits_{G \,:\, j \in G} \big( \mu_j - \mu_{j_{\mathrm{worst}}(G)}\big)$ is the minimum distance between (the mean reward of) $j$ and the worst arm $j_{\mathrm{worst}}(G)$ in any of the groups containing $j$. This gap determines when $j$ can be removed (i.e., no longer pulled) if it is not a worst arm in any group.
    \item $\Delta_j'' := \min\limits_{G \,:\, j \in G}\Delta_G$ is the minimum distance between the worst arm in the optimal group $G^*$ and the worst arm in any of the groups containing $j$. This gap determines when all the groups that $j$ is in can be ruled out as being suboptimal. If $j$ is not in the optimal group $G^*$, the removal of these groups also amounts to $j$ being removed, whereas if $j$ is in $G^*$, this value becomes zero.
    \item $\Delta_0 := \min\limits_{G \,:\, G \neq G^*} \Delta_G$ is a fixed constant indicating the distance between the worst arm in the optimal group $G^*$ and the best among the worst arms in the remaining suboptimal groups. This gap determines when the optimal group is found (and the algorithm terminates).
\end{itemize}

Following the definitions above, we define the overall gap associated with each arm $j$ as follows:
\begin{equation}  \label{def:gap}
    \Delta_j = \max\{\Delta_j', \Delta_j'', \Delta_0 \} > 0.
\end{equation}
In \refsec{sec:se}, we will present an algorithm such that, with high probability, arm $j$ stops being sampled after a certain number of pulls dependent on $\Delta_j$.

\subsection{Failure of Naive Approach} \label{sec:failure}

A simple algorithm to solve the max-min grouped bandit problem is to treat it as a combination of $|{\cal G}|$ worst arm search problems. We can consider each group separately, and identify the worst arm for each group via a ``best''-arm identification algorithm (trivially adapted to find the {\em worst} arm instead of the best) such as UCB or LUCB \cite{jamieson2014best}.  We can then rank the worst arms in the various groups to find the optimal group. 
    
However, this method may be highly suboptimal, as it ignores the comparisons between arms from different groups during the search. For instance, consider a setting in which ${\cal G} = \{G_1, G_2\}$ with $G_1 = \{1, \dotsc, k\}$ and $G_2 = \{k+1,\dotsc,n\}$. Suppose that $\mu_1 = 1> ...> \mu_{k-1} = 0.9> \mu_k = 0.8$ and $\mu_{k+1} = 0.1>...> \mu_{n-1} = 0.01 > \mu_n = 0.00999$. We observe that finding the worst arm in $G_2$ is highly inefficient due to the narrow gap of $0.01 - 0.00999$ between arms $n-1$ and $n$. On the other hand, a simple comparison between the observed values of arms from $G_1$ and $G_2$ can relatively quickly verify that all arms in $G_1$ are better than all arms in $G_2$, without needing to know the precise ordering of arms within either group. % This negligence of inter-group comparison is be addressed by our definition of the fundamental gap and the successive elimination algorithm introduced later.

\subsection{Auxiliary Results} \label{sec:aux}

As is ubiquitous in MAB problems, our analysis relies on confidence bounds.  Despite our distinct objective, our setup still consists of regular arm pulls, and accordingly, we can utilize well-established confidence bounds for stochastic bandits.  Many such bounds exist with varying degrees of simplicity vs.~tightness, and to ease the exposition, we do not seek to optimize this trade-off, but instead focus on one representative example given as follows.

\begin{lemma} \label{lem:conf}
    {\em (Law of Iterated Logarithm \cite{jamieson2014best})} 
    Let $Z_1, Z_2, \dotsc$ be i.i.d sub-Gaussian random variables with mean $\mu \in \mathbb{R}$ and parameter $\sigma \leq \frac{1}{2}$.  For any $\epsilon \in (0,1)$ and $\delta \in (0, \frac{1}{e} \log(1+\epsilon))$, with probability at least $1 - \frac{2 + \epsilon}{\epsilon/2}(\frac{\delta}{\log(1+\epsilon)})^{1+\epsilon}$, we have
        \begin{equation}
            \left|\frac{1}{t} \sum\limits_{s = 1}^{t} Z_s - \mu \right| \leq U(t, \delta) \;\; \forall t \geq 1,
        \end{equation}
        where
        \begin{equation}
            U(t, \delta) = (1+ \sqrt{\epsilon}) \sqrt{\frac{1 + \epsilon}{2t} \log\frac{\log(1 + \epsilon)t}{\delta}}. \label{eq:U_def}
        \end{equation}
\end{lemma}

In accordance with this result, we henceforth assume that $R \le \frac{1}{2}$ in \refas{as:noise}, which notably always holds for Bernoulli rewards.

Since the error probability is dependent on the entire set of $n$ arms, we further replace $\delta$ by $\frac{\delta}{n}$ in \reflem{lem:conf} and apply a union bound, which leads to the following upper/lower confidence bound of arm $j$ in round $t$:
\begin{align}
    \mathrm{UCB}_t(j) &= \hat{\mu}_{j, T_j(t)} + U\left(T_j(t), \frac{\delta}{n}\right) \label{def:UCB}\\
    \mathrm{LCB}_t(j) &= \hat{\mu}_{j, T_j(t)} - U\left(T_j(t), \frac{\delta}{n}\right). \label{def:LCB}
\end{align}
Hence, with probability at least $1-\frac{2 + \epsilon}{\epsilon/2}(\frac{\delta}{\log(1+\epsilon)})^{1+\epsilon}$,
\begin{align}
    \mathrm{LCB}_t(j) \leq \mu_j \leq \mathrm{UCB}_t(j), \; \forall j \in \{1, ...,n\}, t \geq 1.
\end{align}
To derive the performance bounds for our algorithms, we further require the following lemma:
\begin{lemma} \label{lem:inversion}
    {\em (Inversion of $U(t, \delta)$ \cite{jamieson2014best}) }
    For any $\epsilon \in (0,1)$, $\delta \in (0, \frac{1}{e} \log(1+\epsilon))$, and $\Delta \in (0,1)$, we have
    \begin{equation} \label{eq:interval}
        \min\left\{k: U\left(k, \frac{\delta}{n}\right) \leq \frac{\Delta}{4}\right\} \leq \frac{2\gamma}{\Delta^2} \log \frac{2\log(\gamma(1+ \epsilon)\Delta^{-2})}{\delta/n}
    \end{equation}
    where
    $\gamma = 8(1+\sqrt{\epsilon})^2(1 + \epsilon)$.
\end{lemma}

\section{Successive Elimination} \label{sec:se}

Elimination-based algorithms are widely used in MAB problems.  In the standard best-arm identification setting, the idea is to sample arms in batches and then eliminate those known to be suboptimal based on confidence bounds, until only one arm remains.  In the max-min setting that we consider, we can use a similar idea, but we need to carefully consider the conditions under which an arm no longer needs to be pulled.  We proceed by describing this and giving the relevant definitions.
    
We will work in epochs indexed by $i$, and let $t_i$ denote the number of arm pulls up to the start of the $i$-th epoch.  For each group $G$, we define the set of potential worst arms as
\begin{align} \label{def:m_i}
    m_i^{(G)}:= \Bigg\{j \in G: \mathrm{LCB}_{t_i}(j) \leq \min\limits_{j' \in G} \mathrm{UCB}_{t_i}(j')\Bigg\}.
\end{align}
This definition will allow us to eliminate arms that are no longer potentially worst in any group.  We additionally consider a set of candidate {\em potentially optimal} groups ${\cal C}_i$, initialized ${{\cal C}_0} = \cal G$ and subsequently updated as follows:
\begin{align} \label{def:C_i}
	\begin{split}
    {\cal C}_{i+1} := &\Bigg\{G \in {\cal C}_i: \ \min\limits_{j' \in m_i^{(G')}} \mathrm{LCB}_{t_i}(j') \leq \\
                              &\min\limits_{j\in m_i^{(G)}} \mathrm{UCB}_{t_i}(j), \ \forall G' \in {\cal C}_i\Bigg\}.
    \end{split}
\end{align}
This definition allows us to stop considering any groups that are already believed to be suboptimal.

Finally, the set of candidate arms (i.e., arms that we still need to continue pulling) is given by
\begin{align} \label{def:A_i}
    {\cal A}_i := \; \{j: j \in m_i^{(G)} \;  \text{for at least one } G \in {\cal C}_i\}.
\end{align}
With these definitions in place, pseudo-code for the successive elimination algorithm is shown in \refalg{alg:elimination}.

\begin{algorithm}[!t]
    \caption{Successive Elimination algorithm}\label{alg:elimination}
    \begin{algorithmic}[1]
        \Require~~ Arms $(a_1, ..., a_n)$, set of groups $\cal G$, parameters $\delta, \; \epsilon > 0$
        \State Initialize $ i = 0, t = 0$ and $T_j(t) = 0$ for all $j$
        \State Set $m_0^{(G)} = G \;$ for all $G \in {\cal G}$, ${{\cal C}_0} = {\cal G}, {\cal A}_0 = \{1,2,...,n\} $
        \While {$|{\cal C}_i| > 1$}
        \State Pull every arm $j$ in ${\cal A}_i$ once, incrementing $t$ after each pull and updating all $T_j(t)$
        \State Compute $m_{i+1}^{(G)}, \; {\cal C}_{i+1}$ and ${\cal A}_{i+1}$ via expressions {\refeq{def:m_i}}, {\refeq{def:C_i}} and {\refeq{def:A_i}}
        \State Increment round index $i$ by 1
        \EndWhile 
        \State \textbf{return} $\hat{\cal C} = {\cal C}_i$
    \end{algorithmic}
\end{algorithm}

We now state our main result regarding this algorithm.

\begin{thm} \label{thm:ub_se}
    {\em (Upper Bound for Successive Elimination)}
    For any max-min grouped bandit instance as defined in \refsec{sec:setup}, given $\epsilon \in (0,1)$ and $\delta \in (0, \frac{1}{e} \log(1+\epsilon))$, with probability at least $1-\frac{2 + \epsilon}{\epsilon/2}\big(\frac{\delta}{\log(1+\epsilon)}\big)^{1+\epsilon}$,  \refalg{alg:elimination} identifies the optimal group and uses a number of arm pulls satisfying
    \begin{align}
        T(\delta, \epsilon) \leq \sum\limits_{j = 1}^n  \frac{2\gamma}{{\Delta_j}^2} \log \frac{2\log(\gamma(1+ \epsilon){\Delta_j}^{-2})}{{\delta}/n}, \label{eq:T_total}
    \end{align}
    where $\gamma = 8(1+\sqrt{\epsilon})^2(1 + \epsilon)$.
\end{thm}

The proof is given in Appendix A, and is based on considering the gap $\Delta_j$ associated with each arm; we show that after the confidence width falls below $\frac{\Delta_j}{4}$, any such arm will be eliminated (or the algorithm will terminate), as long as the confidence bounds are valid.  Applying \reflem{lem:inversion} and summing over the arms then gives the desired result.

While the error term $\delta_0 = \frac{2 + \epsilon}{\epsilon/2}\big(\frac{\delta}{\log(1+\epsilon)}\big)^{1+\epsilon}$ in \refthm{thm:ub_se} is somewhat complicated, one can fix $\epsilon = \frac{1}{2}$ (say) and solve for $\delta$, and it readily follows that the right-hand side of \refeq{eq:T_total} has the standard $O\big( \log\frac{1}{\delta_0}\big)$ dependence.

\section{A Variant of {\scshape StableOpt}} \label{sec:stable}
    
In this section, we first discuss how the max-min grouped bandit problem is related to the problem of adversarially robust optimization. We then demonstrate that a robust optimization algorithm known as {\scshape StableOpt} \cite{bogunovic2018adversarially} can be adapted to our setting with instance-independent regret guarantees.

\paragraph{Connection to adversarially robust optimization} In general, adversarially robust optimization problems take the form $\max\limits_{x \in D_x} \min\limits_{c \in D_c} f(x,c)$, where $x$ is chosen by the algorithm, and $c$ can be viewed as being chosen by an adversary. 

The main focus in \cite{bogunovic2018adversarially} is finding an $\epsilon$-stable optimal input  for some function $f$:
\begin{align}
	x^*_{\epsilon} \in \argmax\limits_{x\in D_x}\min\limits_{\delta \in \Delta_{\epsilon}(x) }f(x+\delta), \label{eq:x*_eps}
\end{align}
where $\Delta_{\epsilon}(x) = \{x' - x: x\ \in D_x \;\text{and} \; d(x,x') \leq \epsilon\}$ is the perturbed region around $x$, and $d(\cdot,\cdot)$ is a generic ``distance'' function (but need not be a true distance measure).

In Appendix D, we discuss a partial reduction to a grouped max-min problem presented in \cite{bogunovic2018adversarially}, while also highlighting the looseness in directly applying the results therein to our setting.

% notes that a setup of form \eqref{def:maxmin_G} applies to disjoint groupings. We proceed by showing that with a 0-1 valued kernel, the form also applies to groupings with overlaps. Furthermore, by the nature of our setup, we can apply {\scshape StableOpt} to our setup with proper choice of kernel and transformation such that the expensive posterior update in Gaussian Process can be replaced by a simpler confidence interval update.
\paragraph{Adapting {\scshape StableOpt}} The original {\scshape StableOpt} algorithm in \cite{bogunovic2018adversarially} corresponding to the problem \refeq{eq:x*_eps} selects $x_t = \argmax\limits_{x\in D_x}\min\limits_{\delta \in \Delta_{\epsilon}(x) }{\rm UCB}_{t - 1}(x+\delta)$, where $\delta_t = \argmin\limits_{\delta \in \Delta_{\epsilon}(x_t) }{\rm LCB}_{t - 1}(x_t+\delta)$, for suitably defined confidence bounds ${\rm UCB}_t$ and ${\rm LCB}_t$.  When adapted to our formulation \refeq{def: G*}, the algorithm becomes the following:
\begin{gather}
G_t = \argmax\limits_{G\in {\cal G}}\min\limits_{j \in G }\mathrm{UCB}_{t-1}(j) \label{def:SO_G_t}\\
j_t = \argmin\limits_{j \in G_t }\mathrm{LCB}_{t-1}(j), \label{def:SO_j_t}
\end{gather}
and the algorithm samples arm $j_t$ in round $t$.

Intuitively, this criterion selects the optimistic estimate of the best group and its pessimistic estimate for the worst arm via the $\mathrm{UCB}$ and $\mathrm{LCB}$ values computed in each round.  Instead of using the general RKHS-based confidence bounds in \cite{bogunovic2018adversarially}, we use those in \refeq{def:UCB}--\refeq{def:LCB}.

The method for breaking ties in \refeq{def:SO_G_t}--\refeq{def:SO_j_t} does not impact our analysis.  For instance, one could break ties uniformly at random, or one may prefer to break ties in \refeq{def:SO_G_t} by taking the group that attains the lower LCB score in \refeq{def:SO_j_t}.

\paragraph{Instance-independent regret bound} Deriving instance-dependent regret bounds for {\scshape StableOpt} appears to be challenging, though would be of interest for future work.  We instead focus on {\em instance-independent} bounds.  Since there always exist instances for which finding the best group requires an arbitrarily long time (e.g., see \refsec{sec:lower}), we instead measure the performance using the simple regret, whose definition is repeated from \refeq{def: simple_regret} as follows:

\begin{align}
r(G^{(T)}) = \max\limits_{G \in {\cal G}} \min\limits_{j \in G} \mu_j - \min\limits_{j \in G^{(T)}} \mu_j,
\end{align}
where $T$ is the time horizon, and $G^{(T)}$ is the group returned after $T$ rounds.  For {\scshape StableOpt}, the theoretical choice of $G^{(T)}$ is given by \cite{bogunovic2018adversarially}
\begin{align}
G^{(T)} = G_{t^*}, \quad t^* = \argmax\limits_{t \in \{1,..., N\}} \min\limits_{j \in G_t} \mathrm{LCB}_{t-1}(j). \label{eq:choice_G}
\end{align}
Here and subsequently, we define $\mathrm{LCB}_0(j) = 0$ and $\mathrm{UCB}_0(j) = 1$ in accordance with the fact that the arm means are in $[0,1]$, and for later convenience we similarly define $T_j(0) = 0$ and $U(0,\delta) = 1$.

With these definitions in place, we have the following result; we state a simplified form with fixed $\epsilon$ and an implicit constant factor, but give the precise constants in the proof.
% In order to analyze its variant's performance in our problem setup, we convert the goal of regret minimization to best group identification by noting the fact that $G^{(T)} \equiv G^*$ when the regret $r_{\cal G}(G^{(T)}) < \Delta_0$. Consequently, we obtain a general upper bound on sampling complexity dependent on $\Delta_0$ and instance size $n$ for this algorithm as follows:

\begin{thm}
{(\em Instance-Independent Regret Bound)} \label{thm:ub_so}
Suppose that \refas{as:noise} holds. Given $\delta \in (0, \frac{\log 2}{e})$, the above variant of {\scshape StableOpt} yields with probability at least $1 - O(\delta)$ that
\begin{align}
    r(G^{(T)}) &\le O\bigg( \sqrt{\frac{n}{T}} \Big( \sqrt{\log\frac{n}{\delta}} + \log \log T \Big)  \bigg). % \\
        % &= O^*\bigg( \sqrt{\frac{n\log\frac{n}{\delta}}{T}}  \bigg),
\end{align}
% where $O^*(\cdot)$ hides the $\log \log T$ factor.

% set $C(n, \delta, \epsilon) = O(n(1+\sqrt{\epsilon}) \sqrt{1+\epsilon}\sqrt{\log(\frac{n}{\delta})})$, the {\scshape StableOpt} Variant Algorithm requires at most $N = O\left(\frac{C(n, \delta, \epsilon)^2}{\Delta_0^2}\right)$ rounds to identify the true optimal group with a guarantee $P_e \leq \frac{2 + \epsilon}{\epsilon/2}\left(\frac{\delta}{\log(1+\epsilon)}\right)^{1+\epsilon}$.
\end{thm}

The proof is given in Appendix B, and is based on initially following the max-min analysis of \cite{Bog17} to deduce an upper bound of $\frac{1}{T} \sum\limits_{t = 1}^T 2 \, U(T_{j_t}(t-1), \frac{\delta}{n})$, but then proceeding differently to further upper bound the right-hand side, in particular relying on \reflem{lem:inversion}.

To compare \refthm{thm:ub_so} with \refthm{thm:ub_se}, it helps to consider the following corollary.

\begin{corollary}
    Under the setup of \refthm{thm:ub_so}, if we additionally have that \refas{as:unique} holds and the gaps defined in \refsec{sec:setup} satisfy $\Delta_j \ge \Delta_{\min}$ for all $j=1,\dotsc,n$ and some $\Delta_{\min} > 0$.  Then, with probability at least $1 -O(\delta)$, the algorithm outputs $G^{(T)} = G^*$ provided that $T \ge \Omega^*\big( \frac{n\log\frac{n}{\delta}}{\Delta_{\min}^2} \big)$ where $\Omega^*(\cdot)$ hides $\log\log(\cdot)$ factors in the argument.
\end{corollary}

This result matches the scaling in \refthm{thm:ub_se} whenever $\Delta_j = \Delta_{\min}$ for all $j$, which can be viewed as a minimax worst-case instance.  Moreover, a standard reduction to finding a biased coin (e.g., \cite{Man04}) reveals that any algorithm must use the preceding number of arm pulls (or more) on worst-case instances, at least up to the replacement of $\log\frac{n}{\delta}$ by $\log\frac{1}{\delta}$; hence, there is no significant room for improvement in the minimax sense.

On the other hand, it is also of interest to understand instances that are not of the worst-case kind, in which case the number of arm pulls given in \refthm{thm:ub_se} can be significantly smaller.  We expect that {\scshape StableOpt} also enjoys instance-dependent guarantees (see \refsec{sec:exp} for some numerical evidence), though \refthm{thm:ub_so} does not show it.

\section{Algorithm-Independent Lower Bound} \label{sec:lower}
    
\subsection{Preliminaries} 

We follow the high-level approach of \cite{kaufmann2016complexity}, and make use of the following assumption adopted therein.

\begin{assumption} \label{as:lower}
    The reward distribution $P_j$ for any arm $j$ belongs to family $\cal P$ of distributions parametrized by their mean $\mu_j \in (0, 1)$. Any two distributions $P_j, P_{j'} \in {\cal P}$ are mutually absolutely continuous, and $D(P_j\|P_{j'}) \to 0$ in the limit as $\mu_{j'}$ approaches $\mu_j$.
\end{assumption}

The following assumption is not necessary for the bulk of our analysis, but will allow us to state our results in a cleaner form that can be compared directly to the upper bounds.

\begin{assumption} \label{as:quadratic}
    There exists a constant $\tilde{C} > 0$ such that, for any arm distributions $P_j$ and $P_{j'}$ having corresponding means $\mu_j$ and $\mu_{j'}$, it holds that $D(P_j\|P_{j'}) \le \tilde{C} (\mu_j - \mu_{j'})^2$.
\end{assumption}

As is well-known from previous works, the above assumptions are satisfied in the case of Gaussian rewards, and also Bernoulli rewards under the additional requirement of means strictly bounded away from zero and one (e.g., $\mu_j \in (0.01,0.99)$ for all $j$).

We use the widely-adopted approach of considering a base instance, and perturbing the arm means (ideally only a small amount) to form a different instance with a different optimal group; see Lemma 4 in the appendix.  An additional technical challenge here is that even if $\cal A$ is identifiable (i.e., satisfies \refas{as:unique}), it can easily happen that the perturbed instance is non-identifiable due to the new max-min arm appearing in multiple groups.  To circumvent this issue, we introduce the following definition for the algorithm's success.

\begin{defn}
    We say that a max-min grouped bandit algorithm is {\em uniformly $\delta$-successful} with respect to a class of instances if it satisfies the following:
    \begin{itemize}
        \item For any identifiable instance in the class, the algorithm almost surely terminates, and returns the max-min optimal group (i.e., $G^*$) with probability at least $1-\delta$.
        \item For any non-identifiable instance in the class, the algorithm may or may not terminate, but has a probability at most $\delta$ of returning a group that is not max-min optimal.
    \end{itemize}
\end{defn}

We note that successive elimination in \refsec{sec:se} satisfies these conditions; in the non-identifiable case, as long as the confidence bounds are valid, the algorithm never terminates.

\subsection{Statement of Result}  

In the following, we let $N_j$ denote the total number of times arm $j$ is pulled.

\begin{thm} \label{thm:lower}
{\em (Algorithm-Independent Lower Bound)}
Consider any algorithm that is uniformly $\delta$-successful with respect to the instances satisfying \refas{as:noise}, \refas{as:lower}, and \refas{as:quadratic}.  Fix any identifiable instance ${\cal A} = (a_1,...,a_n)$ with distributions $(P_1, ..., P_n)$ and a specified grouping ${\cal G} = (G_1, ..., G_m)$.  Then, when the algorithm is run on instance $\cal A$, we have the following:
\begin{itemize}
    \item For each $j \in G^*$, the average number of pulls satisfies 
    \begin{align}
        \mathbb{E}[N_j] \geq \frac{\log\frac{1}{2.4\delta}}{\tilde{C} (\Delta'_j + \Delta_0)^2}, \label{eq:lower1}
    \end{align}
    where $\tilde{C}$ appears in \refas{as:quadratic}, and $\Delta'_j$ and $\Delta_0$ are defined in \refsec{sec:assump}.
    \item For each $G \ne G^*$, we have
    \begin{align}
        &\sum\limits_{j \in G \,:\, \mu_j < \mu_{\mathrm{worst}}(G^*)} \mathbb{E}[N_j(\sigma)] \cdot \tilde{C}(\mu_{j_{\rm worst}(G^*)} - \mu_j)^2 \nonumber \\ & \hspace*{5cm} \geq  \log\frac{1}{2.4\delta}. \label{eq:lower2}
    \end{align}
\end{itemize}
\end{thm}

The proof is given in Appendix C, and is based on shifting the given instance to create a new instance with a different optimal group, and then quantifying the number of arm pulls required to distinguish the two instances.  This turns out to be straightforward when $j \in G^*$, but less standard (requiring multiple arms to be shifted) when $j \notin G^*$.

While \refthm{thm:lower} does not directly state a lower bound on the total number of pulls, we can perform some further steps to deduce such a bound depending on the number of groups-per-arm (which could alternatively be upper bounded trivially by $|{\cal G}|$), stated as follows and proved in Appendix C.

\begin{corollary} \label{cor:lower2}
{\em (Simplified Algorithm-Independent Lower Bound)}
Consider the setup of \refthm{thm:lower}, and suppose that there exists an integer $m > 0$ such that every arm appears in at most $m$ groups.  Then, the expected number of arm pulls is lower bounded by 
\begin{align}
    T_{\rm lower}(\delta) = \sum\limits_{j \in G^*} \frac{\log\frac{1}{2.4\delta}}{\tilde{C} (\Delta'_j + \Delta_0)^2} + \frac{1}{m} \sum\limits_{G \in {\cal G} \setminus \{G^*\} }  \frac{\log\frac{1}{2.4\delta}}{ \tilde{C} \Delta_{G}^2}, \label{eq:T_lower}
\end{align}
where $\Delta_G = \mu_{j_{\mathrm{worst}}(G^*)} - \mu_{j_{\mathrm{worst}}(G)}$.
\end{corollary}
 
In the following section, we discuss the strengths and weaknesses of our bounds in detail.

\section{Discussion} \label{sec:discussion}

\paragraph{Cases with near-matching behavior.} We first note that for $j \in G^*$, the number of pulls {\em of that particular arm} dictated by our upper and lower bounds are near-matching.  This is because any $j \in G^*$ has $\Delta''_j = 0$ and hence $\Delta_j = \max\{\Delta'_j, \Delta_0\}$, which matches $\Delta'_j + \Delta_0$ (see the lower bound) to within a factor of two.

Regarding $j \notin G^*$, it is useful to note that $\Delta''_j = \min_{G \,:\, j \in G} \Delta_G$, and $\Delta_G$ appears in \refcor{cor:lower2}.  Hence, the gaps $\Delta_G$ between worst arms play a fundamental role in both the upper and lower bounds.  However, near-matching behavior is not necessarily observed, as we discuss below.  

As an initial positive case, under the trivial grouping $\Gc = \{ \{1\}, \{2\}, \dotsc, \{n\} \}$, our bounds reduce to near-tight bounds for the standard best-arm identification problem \cite{jamieson2014best,kaufmann2016complexity}, with $\sum_{j=1}^n \frac{1}{\Delta_j^2}$ dependence on the gaps $\{\Delta_j\}$ between suboptimal arms and the optimal arm.

More generally, our upper and lower bounds have near-matching dependencies when both the number of items-per-group and groups-per-item are bounded, say by some absolute constant.  In this case, the second term in \refeq{eq:T_lower} is dictated by $\sum_{G \ne G^*} \frac{1}{\Delta_G^2}$ (since $m$ is bounded), and we claim that the same is true for the contribution of $j \notin G^*$ in the upper bound.  To see this, first note that within each group, the arm with the smallest $\Delta_j$ is the one with the smallest $\Delta'_j $ (the other two quantities $\Delta''_j$ and $\Delta_0$ do not vary within $G$), which is $\jworst(G)$ (having $\Delta'_j = 0$).  Thus, $j = \jworst(G)$ incurs the most pulls in $G$, and has $\Delta_j = \max\{\Delta''_j,\Delta_0\}$.  The definition of $\Delta_0$ combined with $j = \jworst(G)$ imply that this simplifies to $\Delta_j = \Delta_G$.  When we have bounded items-per-group and groups-per-item, it follows that $\sum_{j \notin G^*} \frac{1}{\Delta_j^2}$ reduces to $\sum_{G \ne G^*} \frac{1}{\Delta_G^2}$ up to constant factors, as desired.

\paragraph{Cases where the bounds are not tight.} Perhaps most notably, the lower bound only dictates a minimum {\em total} number of pulls for arms in a given group $G \ne G^*$, whereas the upper bound is based on each {\em individual} arm being pulled enough.  It turns out that we can identify weaknesses in both of these, and it is likely that neither bound can consistently be identified as ``tighter'' than the other.

To see a potential weakness in the upper bound in \refthm{thm:ub_se}, consider an instance with $|{\cal G}| = 2$ and only a single arm $j^*$ in the optimal group $G^*$, and a large number of arms in the suboptimal group. For the arms in the suboptimal group $G_2$, suppose that half of them have a mean reward significantly above that of $j^*$, and the other half have a significantly smaller mean reward.  In this case, it is feasible to quickly identify $G_2$ as suboptimal by randomly selecting a small number of arms (namely, $O\big(\log\frac{1}{\delta}\big)$ of them if we require an error probability of at most $\delta$) and sampling them a relatively small number of times.  Hence, it is not necessary to sample every arm in $G_2$.  On the other hand, our proposed elimination algorithm immediately starts by pulling every arm, which may be highly suboptimal if $|G_2|$ is very large. 

% We can overcome the issue of redundant pulls by adjusting our algorithm so that we do not pull every arm in ${\cal A}_t$ in each iteration but a fraction of it. The adjustment involves randomly pulling a number of candidate arms in each candidate group $G$ proportional to $|m_i^{(G)}|$ and compare with an existing reference group's worst arm in the elimination step.

In contrast, the main looseness is clearly in the lower bound if we modify the above example so that $G_2$ only has {\em one} low-mean arm.  In this case, the total number of arm pulls should clearly increase linearly with the number of arms, but the lower bound in \refthm{thm:lower} does not capture this fact; it only states that both $j^*$ and the low-mean arm in $G_2$ should be pulled sufficiently many times.

\paragraph{Difficulties in obtaining uniformly tight bounds} The examples above indicate that the general max-min grouped bandit problem is connected to problem of {\em good-arm identification} \cite{Kat20}, and that improved algorithms might randomly select subsets of arms within the groups (possibly starting with a small subset and expanding it when further information is needed).  In fact, in the examples we gave, if the mean reward of $j^*$ were to be revealed to the algorithm, the remaining task of determining whether $G_2$ contains an arm with a mean below that of $j^*$ would be exactly equivalent to the problem studied in \cite{Kat20}.  Even this sub-problem required a lengthy and highly technical analysis in \cite{Kat20}, and the difficulty appears to compound further when there are multiple non-overlapping groups, and even further when overlap is introduced.  Thus, we believe that the development of near-matching upper and lower bounds is likely to be challenging in general.

\paragraph{Discussion on identifiability assumptions} Recall from \refas{as:unique} that we assume $G^*$ to be uniquely defined.  In contrast, we do not require a unique worst arm within {\em every} group; multiple such arms only amounts to more ways in which the suboptimal group can be identified as such.

The unique optimal group assumption could be removed by introducing a tolerance parameter $\epsilon$, and only requiring to identify a group whose worst arm mean is within $\epsilon$ of the highest possible.  In this case, any gap values $\Delta_j$ that are below $\epsilon$ get capped to $\epsilon$ in the upper bound in \refthm{thm:ub_se}.  Our lower bounds can also be adapted accordingly.   The changes in the analysis are entirely analogous to similar findings in the standard best-arm identification problem (e.g., see \cite{gabillon2012unified}), so we do not go into detail.

% \paragraph{Further discussion} In \refapp{sec:more_disc}, we further discuss the difficulties in obtaining uniformly tight bounds, and also discuss our identifiability assumptions.

\section{Experiments} \label{sec:exp}

In this section, we present some basic experimental results comparing the algorithms considered in the previous sections.

% that demonstrate the performances of algorithms introduced and tuned with various parameters. In the first experiment we analyse the impact of gap $\Delta$ on the sample complexity for both algorithms at a fixed confidence level. The second experiment illustrates the failure of the naive group-wise UCB algorithm discussed in \refsec{sec:failure} as the total arm pulls used exceeds the amount required by our algorithms significantly. Finally we present a comparison of error rate and total arm pulls required among choices of empirical bounds and optimizer selection criterion for each algorithm. This provides further justification for our experimental setup outlined below.

\begin{figure}
    \centering
    \includegraphics[width=0.42\textwidth]{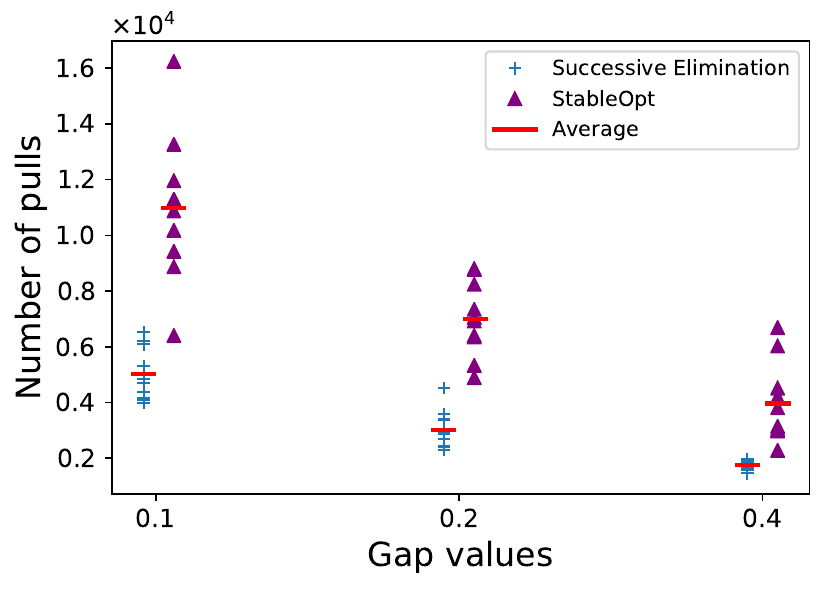} 
    \caption{Plot of total arm pulls used for Successive Elimination and {\scshape StableOpt} for $\Delta \in \{0.1, 0.2, 0.4\}$.}\label{fig:gap_comp}
    \vspace*{-1.5ex}
\end{figure}

\begin{figure}
    \centering
    \includegraphics[width=0.42\textwidth]{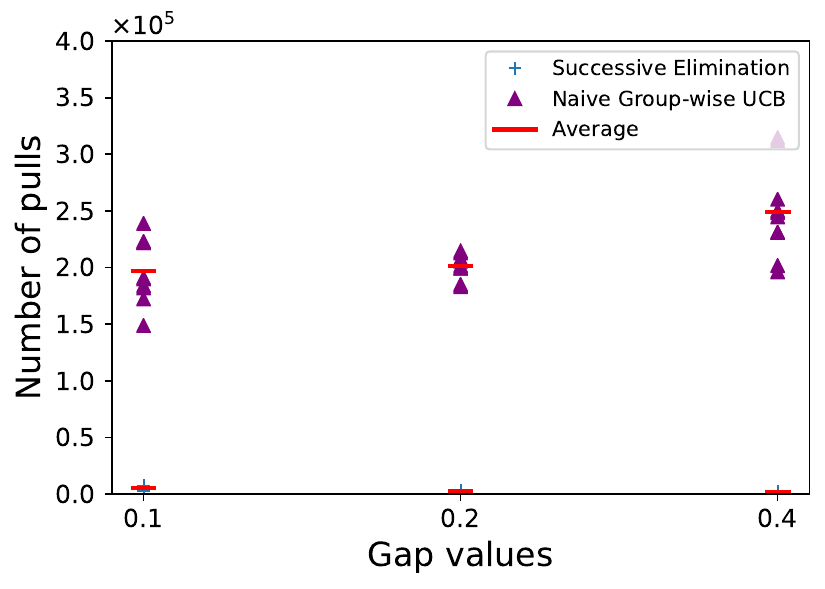} 
    \caption{Comparison of Successive Elimination and the naive group-wise strategy for $\Delta \in \{0.1, 0.2, 0.4\}$.}
    \label{fig:naive_non_overlap}
    \vspace*{-1.5ex}
\end{figure}

\begin{figure}
    \centering
    \includegraphics[width=0.4\textwidth]{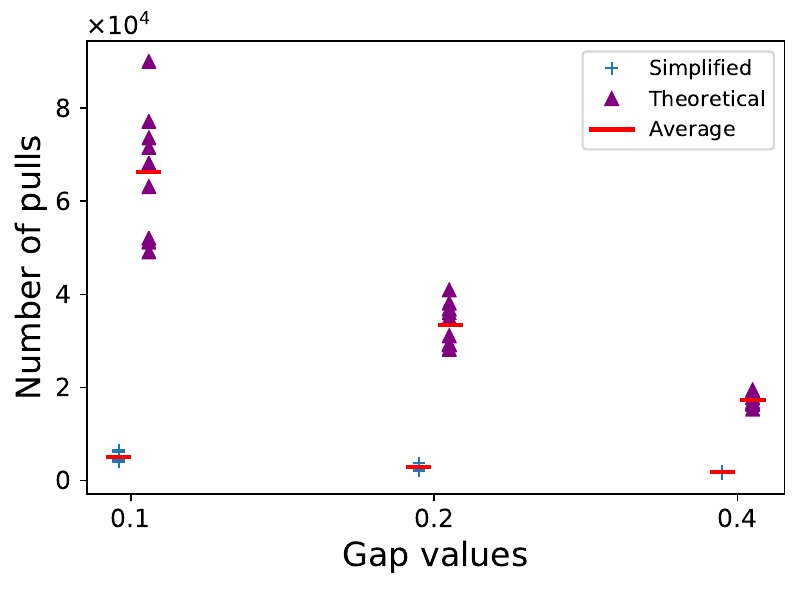} 
    \caption{Comparison of theoretical and simplified choices of confidence bounds.} 		
    \label{fig:bound_comp}
\end{figure}

\begin{figure*}
    \centering
    \subfloat[$\Delta=0.1$]{
        \includegraphics[width=0.33\textwidth]{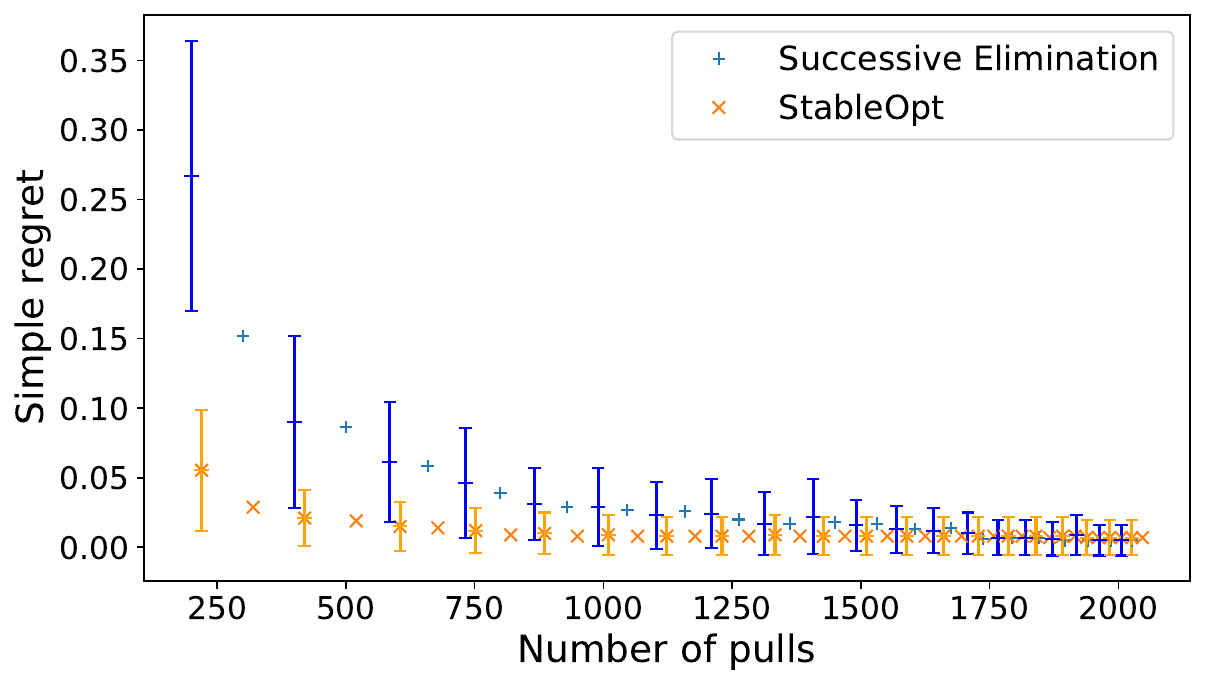}}
    \subfloat[$\Delta=0.2$]{
        \includegraphics[width=0.33\textwidth]{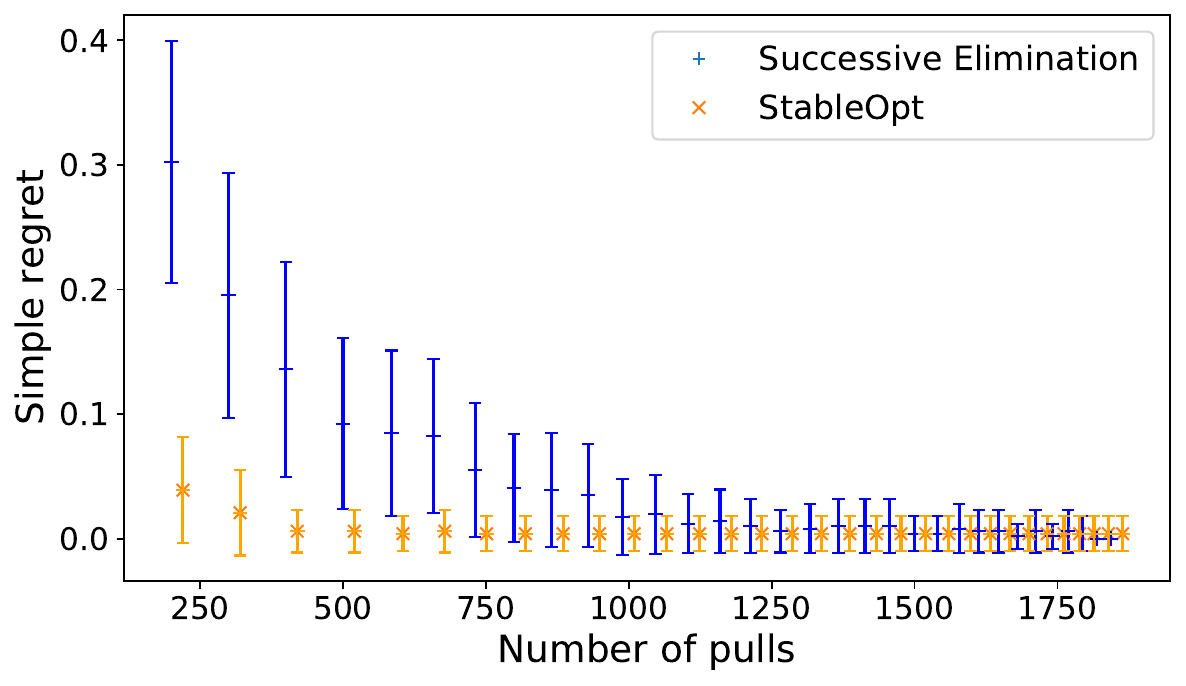}}
    \subfloat[$\Delta=0.4$]{
        \includegraphics[width=0.33\textwidth]{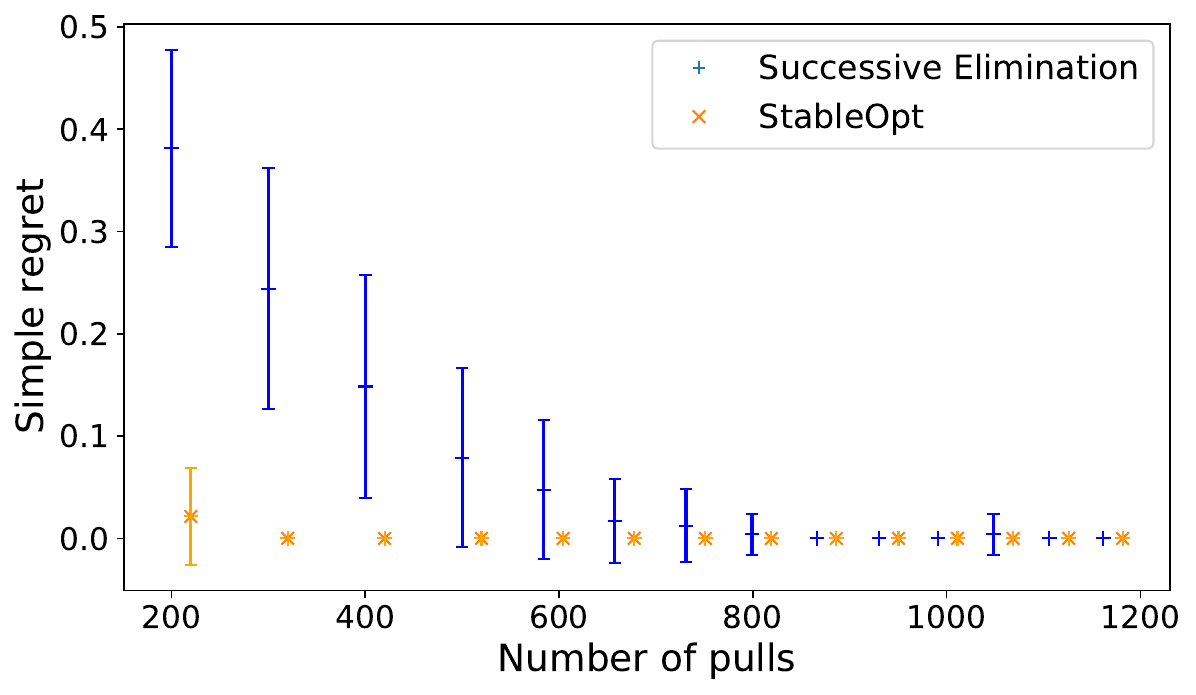}}
    \caption{Simple regret plots with various gap values.} 
    \label{fig:crit_comp}
\end{figure*}

\subsection{Experimental Setup} \label{sec:exp_setup}

In each experiment, we generate 10 instances, each containing 100 arms and 10 possibly overlapping groups.  The arm rewards are Bernoulli distributed, and the instances are generated to ensure a pre-specified minimum gap value ($\Delta$) as per \refeq{def:gap}, and we consider $\Delta \in \{0,1.0.2,0.4\}$.   The precise details of the groupings and arm means are given below.   Empirical error rates (or simple regret values) are computed by performing 10 trials per instance, for a total of 100 trials.

% Each arm $j$ in an instance follow Bernoulli distribution with $\mu_j \in (0,1)$ independently, and is allocated to some groups in the instance with potential duplicates. Each arm is also assigned a maximum gap value ($\Delta$) in the form of  \refeq{def:gap}. Implementation of our algorithms are run on each instance for 10 trials with confidence parameters $\delta = 0.01$ and $\varepsilon = 0.1$.  Empirical error rates and total number of arm pulls for different instances were computed in each trial and summarized into our plots.

%In our implementation, to account for sampling error, we set the stopping criterion for UCB algorithm such that it return an optimal $G$ when $\min_{i \in G} \textrm{LCB}(i) > \max_{G' \in {\cal G} : G' \neq G}\min_{i \in G'}\textrm{UCB}(i) - \epsilon$ where the tolerance value $\epsilon$ is set to be 0.001

% Next, we describe the variants of the algorithms used.

{\bf Details of groups and rewards.} We allocate each arm into each group independently with probability $1/10$, so that each group contains 10 arms on average. For any subsequently non-allocated arms, we assign it to a uniformly random group. In addition, for any empty group, we assign 10 uniformly random arms into it. 

We arbitrary choose the first group to be the optimal one, and set its worst arm $j_0$ to have a mean reward of $0.5$ (here $j_0$ is chosen to ensure that $G^*$ is unique). We further choose the second group to be a suboptimal group whose worst arm $j_1$ meets the gap value exactly, i.e.  $\mu_{j_0} - \mu_{j_1} = \Delta$. Then, we generate other groups' worst arms to be uniform in $[0,\mu_{j_1}]$. Finally, we choose the means for the remaining arms to be uniform in $[\mu_{j_G},1]$, where $j_G$ is the relevant worst arm in the relevant group $G$.

{\bf Confidence bounds.} Both Successive Elimination and {\scshape StableOpt} rely on the confidence bounds \refeq{def:UCB}--\refeq{def:LCB}.  These bounds are primarily adopted for theoretical convenience, so in our experiments we adopt the more practical choice of $\hat{\mu}_{j, T_j(t)} \pm \frac{c}{\sqrt{T_j(t)}}$ with $c = 1$.  Different choices of $c$ are explored in Appendix E.

{\bf Stopping conditions.} Successive Elimination is defined with a stopping condition, but {\scshape StableOpt} is not.  A natural stopping condition for {\scshape StableOpt} is to stop when highest max-min LCB value exceeds all other groups' max-min UCB values.  However, this often requires an unreasonably large number of pulls, due to the existence of UCB values that are only slightly too low for the algorithm to pull based on.   We therefore relax this rule be only requiring it to hold to within a specified tolerance, denoted by $\eta$.  We set $\eta = 0.01$ by default, and explore other values in Appendix E.

We will also explore the notion of simple regret, and to do so, both algorithms require a method for declaring the {\em current best guess} of the max-min optimal group.  We choose to return the group with the best max-min LCB score, though the max-min empirical mean would also be reasonable.

\subsection{Results}

	 {\bf Effect of $\Delta$.} From \refig{fig:gap_comp}, we observe that the number of arm pulls decreases when the gap $\Delta$ increases, particularly for Successive Elimination.\footnote{Error probabilities are not shown, because there were no failures in any of the trials here.} This is intuitive, and consistent with our theoretical results.  These results also suggest that {\scshape StableOpt} can adapt to easier instances in the same way as Successive Elimination; obtaining theory to support this would be interesting for future work.
   
   	 {\bf Comparison to the Naive Approach.} We demonstrate that the simple group-wise approach is indeed suboptimal by comparing its empirical performance with Successive Elimination.  Within each group, we identify the worst arm using the UCB algorithm with the same stopping rule as that of {\scshape StableOpt} described above, and among the arms identified, the one with the highest LCB score is returned.  \refig{fig:naive_non_overlap} supports our discussion in \refsec{sec:failure}, as we observe that this naive approach requires considerably more arm pulls, and does not appear to improve even as $\Delta$ increases. 
        
    {\bf Theoretical Confidence Bounds.} Here we compare our simplified choice of confidence width, $\frac{1}{\sqrt{T_j(t)}}$, to the theoretical choice in \refsec{sec:aux}.  The comparison is given in \refig{fig:bound_comp}, where we observe that the former requires fewer arm pulls and is less prone to runs with an unusually high number of pulls, suggesting that the theoretical choice may be overly conservative.  For both choices, there were no failures (i.e., returning the wrong group) in any of the runs performed.

     {\bf Simple Regret.} As seen above, the total number of pulls comes out to be fairly high for both algorithms.  This is due to stringent stopping conditions, and an investigation of the {\em average simple regret} reveals that the algorithms in fact learn much faster despite not yet terminating, especially for {\scshape StableOpt}, and especially when $\Delta$ is larger; see \refig{fig:crit_comp} (error bars show half a standard deviation).  These results again support the hypothesis that {\scshape StableOpt} naturally adapts to easier instances, though our theory only handles the instance-independent case.
     
     A possible reason why StableOpt attains smaller simple regret in \refig{fig:crit_comp} is that it more quickly steers towards the more promising groups, due to its method of selecting the group with the highest upper confidence bound.  In contrast, Successive Elimination always treats every non-eliminated arm equally, and the simple regret only decreases when eliminations occur.
     
      {\bf Effect of Confidence Width.} In Appendix E, we present further experiments exploring the theoretical choice of confidence width vs.~our practical choice of $\frac{c}{\sqrt{t_j}}$ with $c=1$, as well as considering difference choices of $c$ (and also the {\scshape StableOpt} stopping parameter $\eta$).
%   	\begin{figure}[hbt!]
%		 \centering
%		 \includegraphics[width=0.5\textwidth]{img/naive_overlap.png} 
% 	   	 \caption{In the potentially overlapping grouped bandit instances, (right) the total arm pulls used for naive approach is much higher than that of \refalg{alg:elimination} for most instances with gap value in $[0.05, 0.5]$ and at the same time (left) \refalg{alg:elimination} achieves a lower error rate.}
%  	   	 \label{fig:naive}
%   	\end{figure}

%\begin{figure}[t!]
%    \centering
%    \includegraphics[width=0.42\textwidth]{img/regrets_pulls_c=1_gap=0.1.png} 
%    \caption{Simple regret for $\Delta = 0.1$.} 
%    \label{fig:crit_comp}
%    \vspace*{-1.5ex}
%\end{figure}
%
%\begin{figure}[t!]
%    \centering
%    \includegraphics[width=0.42\textwidth]{img/regrets_pulls_c=1_gap=0.4.png} 
%    \caption{Simple regret for $\Delta = 0.4$.} 
%    \label{fig:crit_comp2}
%    \vspace*{-1.5ex}
%\end{figure}

\section{Conclusion} \label{sec:conclusion}

We have introduced the problem of max-min grouped bandits, and studied the number of arm pulls for both an elimination-based algorithm and a variation of the StableOpt algorithm \cite{bogunovic2018adversarially}.  In addition, we provided an algorithm-independent lower bound, identified some of the potential weaknesses in the bounds, and discussed the difficulties in overcoming them.  

We believe that this work leaves open several interesting avenues for further work.  For instance:
\begin{itemize} 
    \item Following our discussion in \refsec{sec:discussion}, it would be of considerable interest to devise improved algorithms that do not necessarily pull every arm.
    \item Similarly, we expect that it should be possible to establish improved lower bounds that better capture certain difficulties, such as finding a single bad arm in a large group of otherwise good arms.
    \item Finally, one could move from the max-min setting to more general settings in which a single bad arm does not necessarily make the whole group bad, e.g., by considering a given quantile within the group.
\end{itemize}

% not necessary in journal

    % \fontsize{9.0pt}{10.0pt}
    % \selectfont
    
    %\clearpage
    %\newpage

% \newpage
\appendix

%\twocolumn[%
%{\centering \Huge \bf Supplementary Material \par}
%
%\bigskip
%{\centering \large \bf Max-Min Grouped Bandits \par}
%\bigskip\bigskip
%]

% \section{Proofs for Successive Elimination} \label{sec:pfs_se}

\section{Proof of \refthm{thm:ub_se} (Upper Bound for Successive Elimination)} \label{sec:pfs_se}

We first formally state the correctness of the algorithm.

\begin{lemma} \label{lem:conv_SE}
    {\em (Correctness of SE)}
    Suppose that \refas{as:noise} and \refas{as:unique} hold. Given $\epsilon \in (0,1)$ and $\delta \in (0, \frac{1}{e} \log(1+\epsilon))$, with probability at least $ 1 - \frac{2 + \epsilon}{\epsilon/2}(\frac{\delta}{\log(1+\epsilon)})^{1+\epsilon}$, we have that \refalg{alg:elimination} returns $\hat{\cal C} = \{G^*\}$.
\end{lemma}
% The proof is given in \refapp{sec:pfs_se}, and is based on an induction argument. The base case is that $G^*$ is initially a candidate group and $j_{\mathrm{worst}}(G^*)$ is initially a candidate worst arm in $G^*$.  The induction argument then shows that this must remains true in subsequent epochs as long as the confidence bounds are valid.
% \subsection{Proof of \reflem{lem:conv_SE} (Correctness of SE)} \label{sec:pf_conv_SE}
\begin{proof}
    We define an event under which the optimal group $G^*$ remains a candidate group, and its worst arm $j_{\mathrm{worst}}(G^*)$ remains a candidate worst arm in $G^*$ in epoch $i$:
    \begin{align*}
        {\cal E}_i := \{G^* \in {\cal C}_i\} \cap \{j_{\mathrm{worst}}(G^*) \in {\cal A}_i \}.
    \end{align*}
    We show that ${\cal E}_i$ holds for all $i$ whenever the confidence bounds introduced \refsec{sec:aux} are valid, which we know holds with probability at least $1 - \frac{2 + \epsilon}{\epsilon/2}(\frac{\delta}{\log(1+\epsilon)})^{1+\epsilon}$.
    % conditioned on ${\cal E}_k$ holding for $0\leq k \leq i-1$ rounds, i.e,  $\mathbb{P}({\cal E}_i \mid {\cal E}_{i-1}, ..., {\cal E}_1) \geq 1 - \frac{2 + \epsilon}{\epsilon/2}(\frac{\delta}{\log(1+\epsilon)})^{1+\epsilon}$:
    
    At the beginning of \refalg{alg:elimination}, $m_0^{(G)} = G$ for all $G \in \cal G$ and $G^* \in {\cal G} = {\cal C}_0$. Hence $j_{\mathrm{worst}}(G^*) \in  \{1,2,...,n\} =  {\cal A}_0$, and the base case ${\cal E}_0$ holds.  We proceed by showing that when ${\cal E}_{i-1}$ holds, so does ${\cal E}_i$.
    
    First, the validity of the confidence bounds gives for all $G$ that
    \begin{align}
        \mathrm{LCB}_{t_i}(j_{\mathrm{worst}}(G)) &\leq \mu_{j_{\mathrm{worst}}(G)} \\ &= \min\limits_{j' \in G} \mu_{j'}  \\  &\leq \min\limits_{j' \in G} \mathrm{UCB}_{t_i}(j'), \label{eq:m_proof}
    \end{align}
    implying that $j_{\mathrm{worst}}(G) \in  m_{i}^{(G)}$ for all $G$.  That is, each group's truly worst arm always remains a candidate worst arm, and this holds in particular for $j_{\rm worst}(G^*)$ in $G^*$.  For later use, it will also be useful to note that this property implies the final minimum in \refeq{eq:m_proof} can be restricted to $m_i^{(G)}$ instead of $G$, yielding
    \begin{align}
        \min\limits_{j\in m_i^{(G)}} \mathrm{UCB}_{t_i}(j) \geq \mu_{j_{\mathrm{worst}}(G)}. \label{eq: G*_in_C_0}
    \end{align}

    % Using the definition of $\mathrm{UCB}_{t_i}(j')$ from Section \ref{sec:aux}, we have that for all groups $G \in {\cal G}$, if an arm $j \in G$ has $\mu_j >\min_{j'\in m_i^{(G)}} \mathrm{UCB}_{t_i}(j') $, then $\mu_j > \mu_{j'}$ where $j'$ achieves the smallest value of UCB within $m_i^{(G)}$. Hence, $j \neq j_{\mathrm{worst}}(G)$ in this case. Therefore, we can conclude that for all $G \in {\cal G}$:
    It remains to show that $G^*$ always remains a candidate potentially optimal group.   Denoting the set of worst arms among all groups as ${\cal J}_{\rm worst}({\cal G}) := \{j: \mu_j = \min\limits_{j' \in G} \mu_{j'} \text{ for some } G \in {\cal G}\}$, we have
    \begin{align}
        \min\limits_{j\in m_i^{(G^*)}} \mathrm{UCB}_{t_i}(j)
        & \geq \mu_{j_{\mathrm{worst}}(G^*)} \label{eq: G*_in_C_1} \\
        & = \max\limits_{j \in {\cal J}_{\rm worst}({\cal G})} \mu_{j} \label{eq: G*_in_C_2}  \\
        & \geq \max\limits_{j \in {\cal J}_{\rm worst}({\cal G})} \mathrm{LCB}_{t_i}(j) \label{eq: G*_in_C_3}  \\
        & \geq \min\limits_{j \in m_i^{(G)}} \mathrm{LCB}_{t_i}(j) \; \forall G \in {\cal C}_{i} \label{eq: G*_in_C_4} ,
    \end{align}
    where \refeq{eq: G*_in_C_1} is an application of \refeq{eq: G*_in_C_0} to $G^*$, \refeq{eq: G*_in_C_2} holds by the definitions of $G^*$ and ${\cal J}_{\rm worst}({\cal G})$, \refeq{eq: G*_in_C_3} uses the validity of the confidence bounds, and \refeq{eq: G*_in_C_4} holds because for each candidate group $G \in {\cal C}_i$,, we have 
    \begin{align}
        \min\limits_{j \in m_i^{(G)}} \mathrm{LCB}_{t_i}(j) &\leq \mathrm{LCB}_{t_i}(j_{\mathrm{worst}}(G)) \nonumber \\ &\leq \max\limits_{j \in {\cal J}_{\rm worst}({\cal G})} \mathrm{LCB}_{t_i}(j)
    \end{align}
    due to the fact that $j_{\mathrm{worst}}(G) \in {\cal J}_{\rm worst}({\cal G})$ and $j_{\mathrm{worst}}(G) \in m_i^{(G)}$.
    
    By the definition of ${\cal C}_i$ and \refeq{eq: G*_in_C_4}, we conclude that $G^* \in {\cal C}_i$ after the $i$-th epoch. Combining this with $j_{\mathrm{worst}}(G^*) \in  m_{i}^{(G^*)}$ , we then have $j_{\mathrm{worst}}(G^*) \in  {\cal A}_i$, implying that ${\cal E}_i$ holds.
    
    Since the algorithm stops when $|{\cal C}_i| = 1$ and $G^* \in {\cal C}_i$ for all $i$, we conclude that the returned set $\hat{\cal C} = {\cal C}_i$ must contain only the optimal group $G^*$.  Note that by the identifiability assumption (\refas{as:unique}) and the fact that $U(t,\delta) \to 0$ as $t \to \infty$, the algorithm will never continue running forever when the confidence bounds remain valid.
\end{proof}

% \subsection{Proof of \refthm{thm:ub_se} (Upper Bound for Successive Elimination)} \label{sec:pf_ub_se}

Having established high-probability correctness in \reflem{lem:conv_SE}, it remains to bound the number of arm pulls. We bound the number of pulls separately for each arm, considering the cases $j \in G^*$ and $j \notin G^*$ separately, and showing that $U(T_j(t), \frac{\delta}{n}) < \frac{\Delta_j}{4}$ is a sufficient condition for the arm to be eliminated in all cases. We then apply \reflem{lem:inversion} and sum over the arms to obtain the result.

We henceforth suppose that the confidence bounds are valid, as we already considered in the proof of \reflem{lem:conv_SE}. 
First observe that if $U(T_j(t), \frac{\delta}{n}) < \frac{\Delta_j}{4}$, then we have $|\mathrm{UCB}_{t_i}(j) - \mathrm{LCB}_{t_i}(j)| < \frac{\Delta_j}{2}$.  In the following, we assume that this is the case for all $j$ indexing non-eliminated arms; note that by construction, all such arms have been pulled exactly the same number of times after each epoch. 

\paragraph{Case 1 ($j \in G^*$)} In this case, we immediately have $\mu_j \geq \mu_{j_{\mathrm{worst}}(G^*)}$.  By design in the algorithm, $j$ will stop being pulled in either of the following scenarios:
\begin{enumerate}[noitemsep,topsep=0pt,parsep=0pt,partopsep=0pt]
    \item $j$ is no longer a potential worst arm in any group;
    \item $G^*$ is found to be the optimal group, and the algorithm terminates.
\end{enumerate}
Recall the definitions of $\Delta'_j$, $\Delta''_j$, $\Delta_0$, and $\Delta_j$ in \refsec{sec:assump}.  For $j \in G^*$, we have $\Delta_j'' =\mu_{j^*} - \mu_{j^*}= 0$, and hence $\Delta_j = \min \{\Delta'_j, \Delta_0\}$. From the intuition behind each gap value defined in \refsec{sec:setup}, we note that scenario 1 above is related to $\Delta_j'$, and scenario 2 is related to $\Delta_0$. We now consider each scenario separately as follows:
\begin{enumerate}[noitemsep,topsep=0pt,parsep=0pt,partopsep=0pt]
    \item If $\Delta_j = \Delta_j'$, then for all $G$ with $j \in G$, we have
    \begin{align}
        \mathrm{LCB}_{t_i}(j)
        & > \mathrm{UCB}_{t_i}(j) - \frac{\Delta_j'}{2} \label{eq:case_1.1_1} \\
        & \geq \mu_j - \frac{\Delta_j'}{2} \label{eq:case_1.1_2} \\
        & \geq \mu_j - \frac{\mu_j - \mu_{j_{\mathrm{worst}}(G)}}{2} \label{eq:case_1.1_3} \\
        & =\mu_{j_{\mathrm{worst}}(G)} + \frac{\mu_j - \mu_{j_{\mathrm{worst}}(G)}}{2} \label{eq:case_1.1_4} \\
        & \geq \mu_{j_{\mathrm{worst}}(G)} + \frac{\Delta_j'}{2} \label{eq:case_1.1_5} \\
        & \geq \mathrm{LCB}_{t_i}(j_{\mathrm{worst}}(G)) + \frac{\Delta_j'}{2} \label{eq:case_1.1_6} \\
        & > \mathrm{UCB}_{t_i}(j_{\mathrm{worst}}(G)), \label{eq:case_1.1_7}
    \end{align}
    where \refeq{eq:case_1.1_1} and \refeq{eq:case_1.1_7} apply the assumption $|\mathrm{UCB}_{t_i}(j) - \mathrm{LCB}_{t_i}(j)| < \frac{\Delta_j}{2}$,  \refeq{eq:case_1.1_2} and \refeq{eq:case_1.1_6} use the confidence bounds, and \refeq{eq:case_1.1_3} and \refeq{eq:case_1.1_5} use the definition of $\Delta'_j$. From \refeq{eq:case_1.1_7}, we have that $j$ is removed from ${\cal A}_i$ and is no longer pulled.
    \item If $\Delta_j = \Delta_0$, then for $G \neq G^*$, we have
    \begin{align}
        &\mathrm{LCB}_{t_i}(j_{\mathrm{worst}}(G^*)) \nonumber \\
        & ~~ > \mathrm{UCB}_{t_i}(j_{\mathrm{worst}}(G^*)) - \frac{\Delta_0}{2} \label{eq:case_1.2_1} \\
        & ~~ \geq \mu_{j_{\mathrm{worst}}(G^*)} - \frac{\Delta_0}{2} \label{eq:case_1.2_2} \\
        & ~~ \geq \mu_{j_{\mathrm{worst}}(G^*)} - \frac{\mu_{j_{\mathrm{worst}}(G^*)} - \mu_{j_{\mathrm{worst}}(G)}}{2} \label{eq:case_1.2_3} \\
        & ~~ = \mu_{j_{\mathrm{worst}}(G)} + \frac{\mu_{j_{\mathrm{worst}}(G^*)} - \mu_{j_{\mathrm{worst}}(G)}}{2} \label{eq:case_1.2_4} \\
        & ~~ \geq \mu_{j_{\mathrm{worst}}(G)} + \frac{\Delta_0}{2} \label{eq:case_1.2_5} \\
        & ~~ \geq \mathrm{LCB}_{t_i}(j_{\mathrm{worst}}(G)) + \frac{\Delta_0}{2} \label{eq:case_1.2_6} \\
        & ~~ > \mathrm{UCB}_{t_i}(j_{\mathrm{worst}}(G)), \label{eq:case_1.2_7}
    \end{align}
    where \refeq{eq:case_1.2_1} and \refeq{eq:case_1.2_7} apply the assumption $|\mathrm{UCB}_{t_i}(j) - \mathrm{LCB}_{t_i}(j)| < \frac{\Delta_j}{2}$,  \refeq{eq:case_1.2_2} and \refeq{eq:case_1.2_6} use the confidence bounds, and \refeq{eq:case_1.2_3} and \refeq{eq:case_1.2_5} use the definition of $\Delta'_j$.  Then, \refeq{eq:case_1.2_7} implies that all of the non-optimal groups are removed from $\Cc_i$, so the algorithm terminates and $j$ is no longer pulled.
\end{enumerate}

\paragraph{Case 2 ($j \notin G^*$)} In this case, $j$ will stop being pulled if any of the following scenarios are satisfied:
\begin{enumerate}[noitemsep,topsep=0pt,parsep=0pt,partopsep=0pt]
    \item $j$ is no longer a potential worst arm in any group;
    \item $G^*$ is found to be the optimal group, and the algorithm terminates;
    \item All the groups $G$ where $j\in G$ are no longer candidate groups.
\end{enumerate}
The gap values associated with these conditions are $\Delta'_j$, $\Delta_0$, and $\Delta''_j$, respectively.
In our elimination algorithm, $G^*$ is found only after all suboptimal groups are removed. Therefore, the second scenario will never be satisfied before the third condition is satisfied, and we have $\Delta_j = \min\{\Delta'_j,\Delta''_j\}$.

For brevity, we use the shorthand $j^* = j_{\rm worst}(G^*)$ in the remainder of this section.  If the arm $j$ has a mean reward satisfying $\mu_j > \mu_{j^*}$, then 
\begin{align}
    \Delta_j' &= \min\limits_{G \,:\, j \in G} \big( \mu_j - \mu_{j_{\mathrm{worst}}(G)} \big) \nonumber \\ &> 
    \min\limits_{G \,:\, j \in G} \big( \mu_{j^*} - \mu_{j_{\mathrm{worst}}(G)} \big) = \Delta_j'' > 0.
\end{align}
Hence, $\Delta_j \equiv \Delta_j'$. In this case, by the same reasoning as \refeq{eq:case_1.1_1}--\refeq{eq:case_1.1_7}, the first condition is satisfied and $j$ is removed from ${\cal A}_i$.

By the same reasoning, if $\mu_j < \mu_{j^*}$, then $\Delta_j = \Delta_j''$. In this case, for all $G \ne G^*$ with $j\in G$,  we have
\begin{align}
    &\mathrm{LCB}_{t_i}(j_{\mathrm{worst}}(G^*)) \nonumber \\
    & ~~ > \mathrm{UCB}_{t_i}(j_{\mathrm{worst}}(G^*)) - \frac{\Delta_j''}{2} \label{eq:case_2.2_1} \\
    & ~~ \geq \mu_{j_{\mathrm{worst}}(G^*)} - \frac{\Delta_j''}{2} \label{eq:case_2.2_2} \\
    & ~~ \geq \mu_{j^*} - \frac{\mu_{j_{\mathrm{worst}}(G^*)} - \mu_{j_{\mathrm{worst}}(G)}}{2} \label{eq:case_2.2_3} \\
    & ~~ = \mu_{j_{\mathrm{worst}}(G)} + \frac{\mu_{j^*} - \mu_{j_{\mathrm{worst}}(G)}}{2} \label{eq:case_2.2_4} \\
    & ~~ \geq \mu_{j_{\mathrm{worst}}(G)} + \frac{\Delta_j''}{2} \label{eq:case_2.2_5} \\
    & ~~ \geq \mathrm{LCB}_{t_i}(j_{\mathrm{worst}}(G)) + \frac{\Delta_j''}{2} \label{eq:case_2.2_6} \\
    & ~~ > \mathrm{UCB}_{t_i}(j_{\mathrm{worst}}(G)), \label{eq:case_2.2_7} 
\end{align}
where \refeq{eq:case_2.2_1} and \refeq{eq:case_2.2_7} apply the assumption $|\mathrm{UCB}_{t_i}(j) - \mathrm{LCB}_{t_i}(j)| < \frac{\Delta_j}{2}$,  \refeq{eq:case_2.2_2} and \refeq{eq:case_2.2_6} use the confidence bounds, \refeq{eq:case_2.2_3} and \refeq{eq:case_2.2_5} use the definition of $\Delta'_j$, and  \refeq{eq:case_2.2_4} uses the definition of $j^*$. Since \refeq{eq:case_2.2_7} implies the removal of all $G$ where $j\in G$, we obtain that all of these suboptimal groups are eliminated, and hence scenarios 3 is satisfied and $j$ is no longer pulled.

Having handled both cases, we conclude that arm $j$ is no longer pulled when $U(T_j(t), \frac{\delta}{n}) < \frac{\Delta_j}{4}$. Combining this with \reflem{lem:inversion}, we obtain the bound on number of arm pulls $T_j(t)$ for each individual arm $j$:
\begin{align}
    T_j(t) \leq \frac{2\gamma}{{\Delta_j}^2} \log \frac{2\log(\gamma(1+ \epsilon){\Delta_j}^{-2})}{{\delta}/n}. \label{eq:indiv_bound}
\end{align}
Summing over the $n$ arms, we obtain the bound on the total number of arm pulls in \refeq{eq:T_total}.

\section{Proof of \refthm{thm:ub_so} (Regret Bound for \scshape{StableOpt})} \label{sec:pf_ub_so}

The first steps of the proof follow those of \cite{bogunovic2018adversarially}.  With probability at least $1 - \frac{2 + \epsilon}{\epsilon/2}\big(\frac{\delta}{\log(1+\epsilon)}\big)^{1+\epsilon}$, the confidence bounds in \refeq{def:UCB}--\refeq{def:LCB} are uniformly valid, and we henceforth condition on this being the case.  For the group $G_t$ and corresponding arm $j_t \in G_t$ selected in round $t$, we have:
\begin{align}
    r(G_t) & = \max\limits_{G \in {\cal G}} \min\limits_{j \in G} \mu_j -                            \min\limits_{j \in G_t} \mu_j  \label{eq:SO_ub1}\\
    & \leq \max\limits_{G \in {\cal G}} \min\limits_{j \in G} \mu_j - \min\limits_{j \in G_t} \mathrm{LCB}_{t -1}(j) \label{eq:SO_ub2} \\
    & =\max\limits_{G \in {\cal G}} \min\limits_{j \in G} \mu_j - \mathrm{LCB}_{t -1}(j_t) \label{eq:SO_ub3} \\
    & \leq \max\limits_{G \in {\cal G}} \min\limits_{j \in G} \mathrm{UCB}_{t-1}(j) - \mathrm{LCB}_{t -1}(j_t) \label{eq:SO_ub4} \\
    & = \min\limits_{j \in G_t} \mathrm{UCB}_{t-1}(j) -\mathrm{LCB}_{t -1}(j_t) \label{eq:SO_ub5} \\
    & \leq \mathrm{UCB}_{t-1}(j_t) - \mathrm{LCB}_{t-1}(j_t) \label{eq:SO_ub6} \\
    & = 2 \, U\Big( T_{j_t}(t-1), \frac{\delta}{n}\Big) \label{eq:SO_ub7}
\end{align}
where \refeq{eq:SO_ub2} and \refeq{eq:SO_ub4} use the validity of the confidence bounds, \refeq{eq:SO_ub3} and \refeq{eq:SO_ub5} use the selection rules for $j_t$ and $G_t$, and \refeq{eq:SO_ub7} uses the definitions of the confidence bounds.

Using the choice of $G^{(T)}$ in \refeq{eq:choice_G}, we further have:
\begin{align}
    r(G^{(T)}) & \leq \max\limits_{G \in {\cal G}} \min\limits_{j \in G} \mu_j - \mathrm{LCB}_{T -1}(j_T) \label{eq:SO_ub8} \\
    & \leq \frac{1}{T} \sum\limits_{t = 1}^T \Big( \max\limits_{G \in {\cal G}} \min\limits_{j \in G} \mu_j - \mathrm{LCB}_{t -1}(j_t) \Big) \label{eq:SO_ub9} \\
    & \leq \frac{1}{T} \sum\limits_{t = 1}^T 2 \, U(T_{j_t}(t-1), \frac{\delta}{n}), \label{eq:SO_ub10}
\end{align}
where \refeq{eq:SO_ub8} follows from \refeq{eq:SO_ub2}, \refeq{eq:SO_ub9} bounds the minimum by the average, and \refeq{eq:SO_ub10} follows from the argument leading to \refeq{eq:SO_ub7}.

% From \refeq{eq:SO_ub10}, we note that when $\frac{1}{T} \sum\limits_{t = 1}^N2 \times U(T_{j_t}(t-1), \frac{\delta}{n}) < \Delta_0$, we can guarantee that the algorithm selects the true optimal with probability at least $ 1 - \frac{2 + \epsilon}{\epsilon/2}(\frac{\delta}{\log(1+\epsilon)})^{1+\epsilon}$.

We observe from \refeq{eq:U_def} that $U(t, \frac{\delta}{n})$ is monotonically decreasing with respect to $t$.\footnote{An exception may be moving from $t=1$ to $t=2$, but this does not affect our argument here.} Therefore, $\sum\limits_{t = 1}^T U(T_{j_t}(t-1), \frac{\delta}{n}) $ is highest when each arm is pulled the same number of times (up to rounding), i.e., $\sum\limits_{t=1}^T U(T_{j_t}(t-1), \frac{\delta}{n}) \leq n \big(1 + \sum\limits_{t = 1}^{\lfloor \frac{T}{n} \rfloor } U(t, \frac{\delta}{n}) \big)$, where we recall that we define $U(0,\delta) = 1$. Hence:
\begin{align}
    &\frac{1}{T} \sum\limits_{t = 1}^N 2 \, U\Big(T_{j_t}(t-1), \frac{\delta}{n}\Big) \nonumber \\ 
    &~~ \leq \frac{2n}{T} \bigg( 1 + \sum\limits_{t = 1}^{\lfloor \frac{T}{n} \rfloor} U\Big(t, \frac{\delta}{n}\Big) \bigg) \label{eq:SO_new1} \\
    &~~  \leq \frac{2n}{T} +  \frac{2n}{T} \cdot C_1(\epsilon) \sum\limits_{t = 1}^{\lfloor \frac{T}{n} \rfloor} \sqrt{\frac{1}{2t}\log\frac{\log(1 + \epsilon)t}{\frac{\delta}{n}}} \label{eq:SO_new2} \\
    &~~  \leq \frac{2n}{T} +  \frac{2C_1(\epsilon)n}{T} \sum\limits_{t = 2}^{\lfloor \frac{T}{n} \rfloor} \bigg( \sqrt{\frac{1}{2t}\log \frac{n}{\delta} }  \nonumber \\
    & \qquad\qquad\qquad + \sqrt{\frac{1}{2t}\log(\log((1 + \epsilon)t))} \bigg) \label{eq:SO_new3} \\
    &~~  \le \frac{2n}{T} + \frac{2 C_2(n, \delta,\epsilon,T)n}{T} \sum\limits_{t = 1}^{\lfloor \frac{T}{n} \rfloor} \sqrt{\frac{1}{t}} \label{eq:SO_new4} \\
    &~~  \le   \frac{2n}{T} + \frac{4 C_2(n, \delta,\epsilon,T)n}{T} \sqrt{\frac{T}{n}} \label{eq:SO_new5} \\
    &~~  = \frac{2n}{T} + 4 C_2(n, \delta,\epsilon,T) \sqrt{\frac{n}{T}}, \label{eq:SO_new6}
\end{align}
where \refeq{eq:SO_new2} uses the definition of $U$ and defines $C_1(\epsilon) = (1+\sqrt{\epsilon})\sqrt{1+\epsilon}$, \refeq{eq:SO_new3} holds since $\sqrt{a+b} \leq \sqrt{a} + \sqrt{b}$, \refeq{eq:SO_new4} defines the quantity $C_2(n, \delta,\epsilon,T) = \frac{C_1(\epsilon)}{\sqrt 2} \big( \sqrt{\log\frac{n}{\delta}} + \sqrt{\log(\log((1 + \epsilon)T))}$, and \refeq{eq:SO_new5} uses the fact that $\sum_{i=1}^k \frac{1}{\sqrt i} \le 2\sqrt{k}$.

Combining \refeq{eq:SO_ub10} and \refeq{eq:SO_new6} and letting $\epsilon \in (0,1)$ be an arbitrary fixed constant gives the desired $O\big( \sqrt{\frac{n}{T}} \big( \sqrt{\log\frac{n}{\delta}} + \log \log T \big)  \big)$ regret bound. The term $\frac{2 + \epsilon}{\epsilon/2}\big(\frac{\delta}{\log(1+\epsilon)}\big)^{1+\epsilon}$ in the error probability is at most $O(\delta)$ regardless of the choice of $\epsilon \in (0,1)$.

\section{Proofs of Algorithm-Independent Lower Bounds} \label{sec:pfs_lb}

\subsection{A Fundamental Auxiliary Lemma} \label{sec:aux_lemma}

We make use of a fundamental result introduced in \cite{kaufmann2016complexity}, which has subsequently been applied to numerous bandit settings.  The following statement is somewhat different from that in \cite{kaufmann2016complexity}, and the differences are explained in \refsec{sec:note_kauf}.

\begin{lemma} \label{lem:kaufmann}
    {\em (Implicit in \cite{kaufmann2016complexity})}
    Let ${\cal A} = (a_1, ..., a_n)$ and ${\cal A'} = (a_1', ...., a_n')$ be two distinct bandit instances such that for any arm pair $(a_j, a'_j)$, the corresponding distributions $P_j$ and $P'_j$ are mutually absolutely continuous. For any stopping time $\sigma$ which is almost surely finite under instance ${\cal A}$, and any event $\mathcal{E}$ depending only on the reward history up to the stopping time and satisfying $\mathbb{P}_{\cal A}[\mathcal{E}] \in (0,1)$, we have
    \begin{align}
        \sum\limits_{j=1}^n \mathbb{E}_{\cal A}[N_j(\sigma)]D(P_j||P'_j) \geq d(\mathbb{P}_{\cal A}[\mathcal{E}], \mathbb{P}_{\cal A'}[\mathcal{E}]), \label{eq:kauf1}
    \end{align}
    where % the KL divergence $D(P_j\|P_{j'}) = \mathbb{E}_{P_j}\big[\log \frac{P_j}{P_j'}\big]$, and 
    $d(x_1, x_2) = x_1\log\frac{x_1}{x_2} + (1-x_1)\log\frac{1-x_1}{1-x_2}$ is the binary relative entropy function, with $d(0,0) = d(1,1) = 0$, and $\mathbb{P}_{\cal A}$ denotes the probability under instance $\cal A$. 
\end{lemma}

High-probability guarantees for MAB problems are based on attaining a small error probability for a suitably-defined notion of success; in our case, this is the identification of $G^*$.  Hence, if $\mathcal{E}$ is the event that the returned group is the best according to instance $\cal{A}$, then we should have $\mathbb{P}_{\cal A}[\mathcal{E}] \geq 1 - \delta$ and $\mathbb{P}_{\cal A'}[\mathcal{E}] \leq \delta$ given a target error probability $\delta \in \big(0,\frac{1}{2}\big)$, as long as the best group in ${\cal A}$ is not max-min optimal in ${\cal A}'$.  Since $d(x, 1-x) \geq \log \frac{1}{2.4x}$ for all $x \in [0,1]$ \cite{kaufmann2016complexity}, we can then simplify \refeq{eq:kauf1} to
\begin{align}
    \sum\limits_{j=1}^n \mathbb{E}_{\cal A}[N_j(\sigma)]D(P_j||P_{j'}) \geq  \log\frac{1}{2.4\delta}. \label{eq:kauf2}
\end{align}
Then, given a ``base'' instance ${\cal A}$ with optimal group $G^*$, we are left to design another instance ${\cal A'}$ such that $G^*$ is suboptimal, ideally with each $D(P_j||P_{j'})$ being small so that \refeq{eq:kauf2} leads to a stronger lower bound on the number of arm pulls.

\subsection{Note on \reflem{lem:kaufmann}} \label{sec:note_kauf}

\reflem{lem:kaufmann} is slightly different from that in \cite{kaufmann2016complexity}, in that (i) the stopping time is only assumed to be almost-surely finite under $\cal A$ but not necessarily under $\cal A'$, and (ii) we assume that $\mathbb{P}_{\cal A}[\mathcal{E}] \in (0,1)$, rather than allowing all of $[0,1]$.

To understand this difference, we note that in \cite{kaufmann2016complexity}, the almost-sure finite stopping time is used for two purposes: To apply Wald's lemma to a sum of log-likelihood ratios under instance $\cal A$, and to prove that $\mathbb{P}_{\cal A}[\mathcal{E}] = 0 \iff \mathbb{P}_{\cal A}[\mathcal{E}] = 0$ (and similarly if both 0s are replaced by 1s).  The former only requires the stopping time to be almost-surely finite under $\cal A$.  As for the latter, the proof in \cite{kaufmann2016complexity} establishes that if $\sigma$ is almost-surely finite under $\cal A$, then it holds that $\mathbb{P}_{\cal A'}[\mathcal{E}] = 0 \implies \mathbb{P}_{\cal A}[\mathcal{E}] = 0$, or equivalently $\mathbb{P}_{\cal A}[\mathcal{E}] \in (0,1) \implies \mathbb{P}_{\cal A'}[\mathcal{E}] \in (0,1)$.  We do not require the reverse implication, because we already explicitly assume that $\mathbb{P}_{\cal A}[\mathcal{E}] \in (0,1)$.

\subsection{Proof of \refthm{thm:lower}}

As suggested by \reflem{lem:kaufmann}, we prove \refthm{thm:lower} by taking the given instance with optimal group $G^*$, and shifting one or more of its arms (from $\mu_j$ to $\mu'_j$) in a way that ensures that $G^*$ is suboptimal in the new instance. 

%    We construct a new distribution $P_{j}'$ for each arm $a_j \in {\cal A}$ by changing its $\mu_j$ to $\mu_j'$ in the following 2 cases:
%    \begin{itemize}
%        \item If $j \in G^*$ then $\mu_j' = \mu_j - (1+\alpha)\Delta_0, \; \alpha \in (0,1)$;
%        \item If $j \notin G^*$, let $j \in G_w$ where $G_w$ is a suboptimal group, then for each arm $j_w$ in $G_w$ with mean reward less than $\mu_{j_{\mathrm{worst}}(G^*)}$, set $\mu_{j_w}' = \mu_{j_w} + (1+\alpha) (\mu_{j_{\mathrm{worst}}(G^*)} - \mu_{j_{\mathrm{worst}}(G_w)}), \; \alpha \in (0,1)$.
%    \end{itemize}

Without loss of generality, assume that in the original instance, $G_1 =  G^*$ is the optimal group, and $G_2$ is the second best group. We consider the two cases in the theorem statement as follows.

\paragraph{Case 1 ($j \in G_1$)} For a fixed arm $j$, we define an instance ${\cal A}^{(j)}$ such that the arm means $\mu_{i}$ are unchanged for all $i \neq j$, and where $\mu_j$ changes to another value $\mu'_j$; the corresponding distributions are denoted by $P_j$ and $P'_j$.  Specifically, we choose $\mu'_j = \mu_j - (1+\alpha)(\Delta'_j + \Delta_0)$ for some arbitrarily small $\alpha > 0$.  By the definitions of $\Delta'_j$ and $\Delta_0$ in \refsec{sec:assump}, the choice $\alpha = 0$ would make $\mu'_j$ exactly equal to $\mu_{j_{\rm worst}(G_2)}$ (the subtraction of $\Delta'_j$ aligns the mean with $\mu_{j_{\rm worst}(G_1)}$, and the subtraction of $\Delta_0 = \mu_{j_{\mathrm{worst}}}(G_1) - \mu_{j_{\mathrm{worst}}}(G_2)$ further shifts this to $\mu_{j_{\rm worst}(G_2)}$).  Hence, no matter how small $\alpha > 0$, we have that $\mu'_j$ is strictly smaller than $\mu_{j_{\rm worst}(G_2)}$, so that $G_1$ is suboptimal in the new instance.

% We note that this new instance may in principle violate \refas{as:unique}, but as noted in Appendix A of \cite{scarlett2019overlapping}, we can always shift $\alpha$ by an arbitrarily small amount to prevent such a violation.  Thus, we can assume that the new instance is identifiable.

Hence, applying \reflem{lem:kaufmann} with $\cal E$ being the event of outputting $G_1$,\footnote{The condition $\mathbb{P}_{\cal A}[\mathcal{E}] \in (0,1)$ in the lemma is satisfied under our assumptions.  Specifically, \refas{as:lower}, and \refas{as:quadratic} ensure that the algorithm cannot have an error probability of zero.   (The assumption of $\cal A$ being identifiable rules out trivial cases such as only having one group, or all groups being identical.)} we obtain the following lower bound for number of pulls of $j \in G_1$:
\begin{equation}
    \mathbb{E}_{\cal A}[N_j(\sigma)] \geq \frac{\log\frac{1}{2.4\delta}}{D(P_j\|P_j')}
\end{equation}
since $D(P_i \|P_i') = 0$ for all $i \neq j$.  Upper bounding the denominator via \refas{as:quadratic} and using the fact that $\alpha$ can be arbitrarily small, we obtain the desired bound \refeq{eq:lower1}.

\paragraph{Case 2 ($j \notin G_1$).} Let $G$ be any suboptimal group.  Due to the max-min nature of the problem, pushing a {\em single} arm's mean up, even by an arbitrarily large amount, may fail to make $G$ a better group than $G_1$.  Instead, we need to shift {\em all arms with mean at most $\mu_{j_{\rm worst}(G_1)}$} up to a value strictly above $\mu_{j_{\rm worst}(G_1)}$.  To achieve this, we set $\mu'_{j} =  \mu_{j} + (1+\alpha) (\mu_{j_{\mathrm{worst}}(G^*)} - \mu_{j})$ for arbitrarily small $\alpha > 0$.  For any arms in $G$ with mean exactly $\mu_{j_{\mathrm{worst}}(G^*)}$, we can perform an arbitrarily small perturbation similar to Appendix A of \cite{scarlett2019overlapping}.  As a result, $G_1$ is no longer the best group in the new instance.

% In this case, we observe the difficulty of applying same method in Case 1, as pushing $j$ up alone does not necessarily make $G_1$ a suboptimal group. To make $G \neq G_1$ become the new optimal group by modifying arms' distributions in $G$, we can push all arms $j \in G$ whose $\mu_j$ fall below $\mu_{j_{\mathrm{worst}}(G_1)}$ up by an amount sufficient so that all arms in $G$ have mean rewards higher than  $\mu_{j_{\mathrm{worst}}(G_1)}$. Therefore we have a new optimal group in place of $G_1$. This modification corresponds to the second case suggested in the theorem. We refer to this new instance as ${\cal A}^{G}$.  {\color{blue} We note here is that the new optimal group is not necessarily $G$ as we may push a worst arm of another group $G_{w'}$ up during this modification ($j \in G$ and $j \in G_{w'}$ simultaneously). This results in $G_{w'}$ being the new optimal group.}  

Applying \reflem{lem:kaufmann}, we obtain
\begin{align}
    \sum\limits_{j \in G \,:\, \mu_j < \mu_{\mathrm{worst}}(G^*)} \mathbb{E}[N_j(\sigma)] \cdot D(P_j \| P'_{j}) \geq  \log\frac{1}{2.4\delta}, \label{eq:lower2a}
\end{align}
where $P_j$ and $P'_j$ are the distributions in the original and modified instances.  Applying \refas{as:quadratic} and the fact that $\alpha$ is arbitrarily small, we obtain the lower bound \refeq{eq:lower2} on the total number of arm pulls within group $G$.

\subsection{Proof of \refcor{cor:lower2}}

The first term in \refeq{eq:T_lower} follows immediately by summing over $j \in G^*$ in the first case of \refthm{thm:lower}, so it remains to establish the second term.

By the definition $\Delta_G = \mu_{j_{\mathrm{worst}}(G^*)} - \mu_{j_{\mathrm{worst}}(G)}$, the inequality \refeq{eq:lower2} gives for any $G \ne G^*$ that
\begin{align} 
    \sum\limits_{j \in G} \mathbb{E}_{\cal A}[N_j(\sigma)] &\geq \sum\limits_{j \in G \,:\, \mu_j < \mu_{\mathrm{worst}}(G^*)} \mathbb{E}_{\cal A}[N_j(\sigma)]  \\ &\geq \frac{\log\frac{1}{2.4\delta}}{ \tilde{C} \Delta^2_G }, \label{eq:gw_lb}
\end{align}
since the definition of $\Delta_G$ ensures that all gaps appearing in \refeq{eq:lower2} are at most $\Delta_G$.

We observe that \refeq{eq:gw_lb} provides a group-wise lower bound. In a disjoint grouping setup, a simple summation over each group-wise lower bound produces a valid lower bound on total arm pulls for the instance $\cal A$. However, in an instance with overlapping groups, we cannot simply sum the group-wise lower bounds in this way. This is because the overlaps between groups can cause potential double (or triple, etc.) counting of $\mathbb{E}_{\cal A}[N_j(\sigma)]$ for some arms $j$ in the summation. % For example, if $j \in G$ and $j \in G'$, a simple summation of group-wise lower bound using \eqref{eq:gw_lb} may lead to $\mathbb{E}_{\cal A}[N_j(\sigma)]$ counted in both the lower bound for $G$ and the lower bound for $G'$.

To resolve this issue, we use the assumption that each arm can be in at most $m$ groups (with $m=1$ amounting to disjoint groups). Dividing the group-wise bound by $m$ accounts for any potential multiple-counting when computing the lower bound on total arm pulls upon adding up the group-wise bounds.  Thus, we can weaken \refeq{eq:gw_lb} to
\begin{align}
    \sum\limits_{j \in G} \mathbb{E}_{\cal A}[N_j(\sigma)] \geq \frac{1}{m} \sum\limits_{G: G\neq G^*} \frac{\log\frac{1}{2.4\delta}}{ \tilde{C} \Delta^2_G },
\end{align}
with the important difference that it is now valid to further sum over groups; doing so gives the second term in \refeq{eq:T_lower} as desired.

\section{Note on the Original Version of {\scshape StableOpt}} \label{sec:disc_SO}

Recall that the general {\scshape StableOpt} formulation is given in \refeq{eq:x*_eps}.
A connection between \refeq{eq:x*_eps} and a certain grouped max-min problem was already discussed in \cite{bogunovic2018adversarially}, focusing on non-overlapping groups.  In particular, it was noted that the interplay between $x$ and $\delta$ does not need to correspond to addition, and accordingly, we can replace $(x,\delta)$ by $(G,j)$ and transform \refeq{eq:x*_eps} as follows:
\begin{equation}
    G^* \in \argmax\limits_{G\in {\cal G}}\min\limits_{j \in G} f(j). \label{def:maxmin_G}
\end{equation}
In our setting, we take $f(j) = \mu_j$, i.e., the mean of the arm.

The theory in \cite{bogunovic2018adversarially} assumed that $f(x)$ has a bounded norm in a Reproducing Kernel Hilbert Space (RKHS) corresponding to some kernel function $k(x,x')$.  To produce our setting with independent arms, we can choose the 0-1 kernel $k(j,j') = \boldsymbol{1}\{j  = j'\}$, and the RKHS norm reduces to $\|f\|_k = \sqrt{\sum_{i=1}^n \mu_j^2} \le \sqrt{n}$.

While we can apply the main result of \cite{bogunovic2018adversarially} to deduce an instance-independent $O\big( \frac{1}{\sqrt T} \big)$ bound on the regret after $T$ arm pulls, the dependence of the implied constants on the number of arms $n$ is highly suboptimal.  This is because both the squared RKHS norm $\|f\|_k^2$ and the fundamental {\em information gain} quantity in \cite{bogunovic2018adversarially} scale linearly with $n$.  Fortunately, we can sharpen the dependence of the regret on $n$ by suitably adapting the analysis in a manner more directly targeted at our setup, as detailed in \refsec{sec:stable} and \refapp{sec:pf_ub_so}.

\section{Further Experiments} \label{sec:further}

\begin{figure*}[t!]
    \centering
    \subfloat[$c=1,\eta=0.01$]{
        \includegraphics[width=0.25\textwidth]{img/gap_value_0819_c=1_conf=0.01.pdf}}
    \subfloat[$c=1.25,\eta=0.01$]{
        \includegraphics[width=0.25\textwidth]{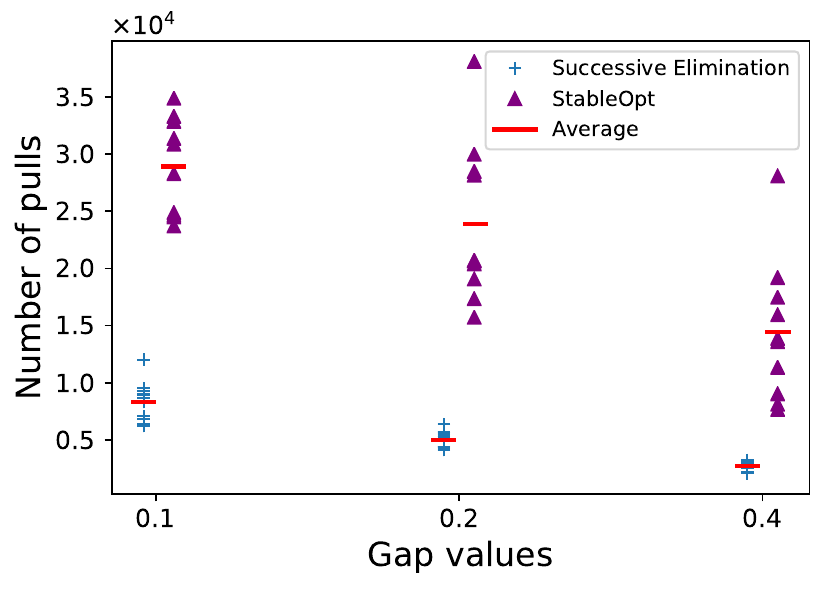}}
    \subfloat[$c=1.5, \eta=0.01$]{
        \includegraphics[width=0.24\textwidth]{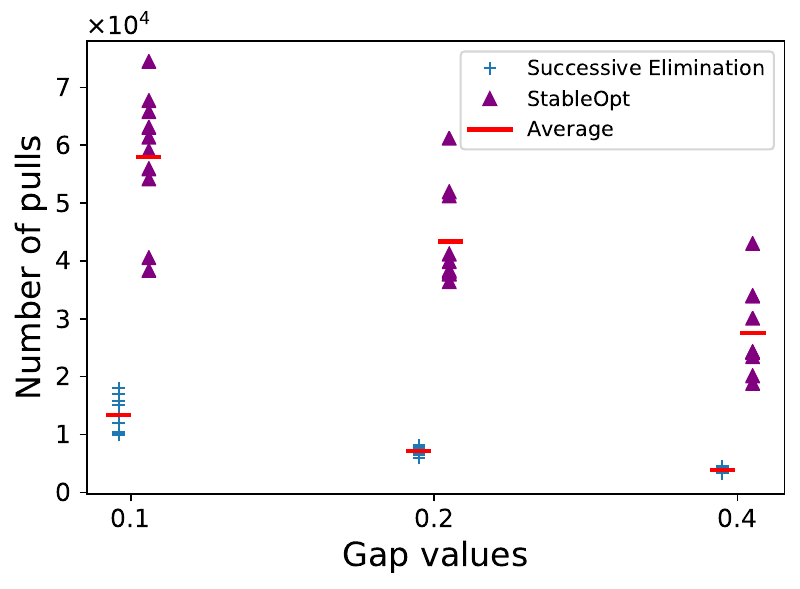}} 
    \subfloat[$c=2,\eta=0.01$]{
        \includegraphics[width=0.25\textwidth]{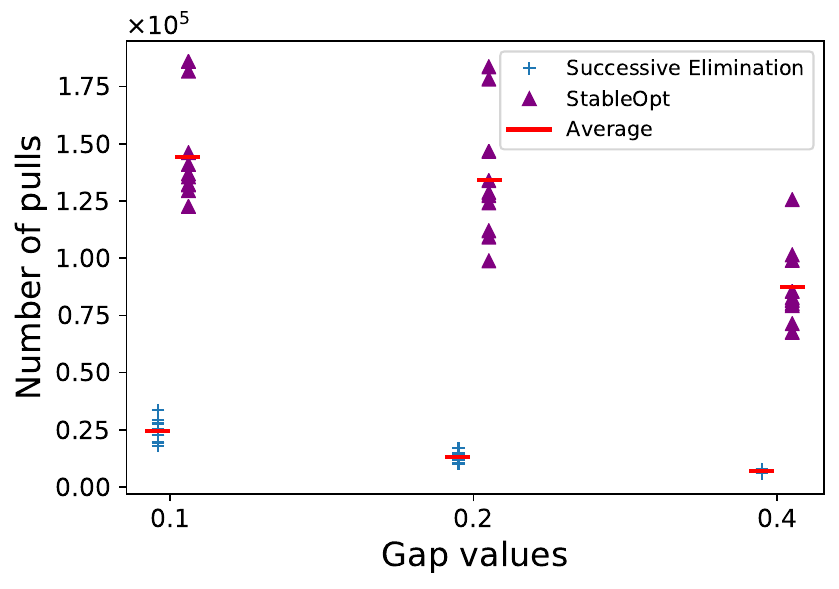}}\\
    \subfloat[$c=1,\eta=0.02$]{
        \includegraphics[width=0.25\textwidth]{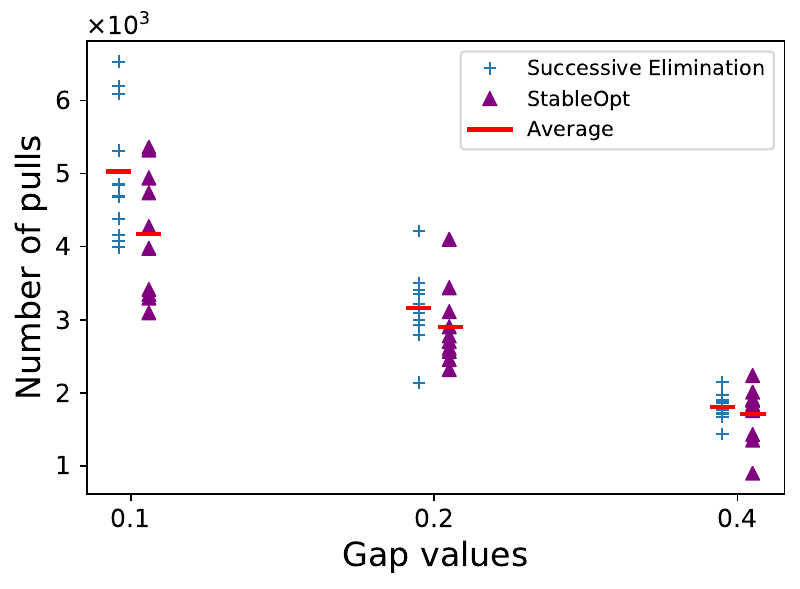}}
    \subfloat[$c=1.25,\eta=0.02$]{
        \includegraphics[width=0.25\textwidth]{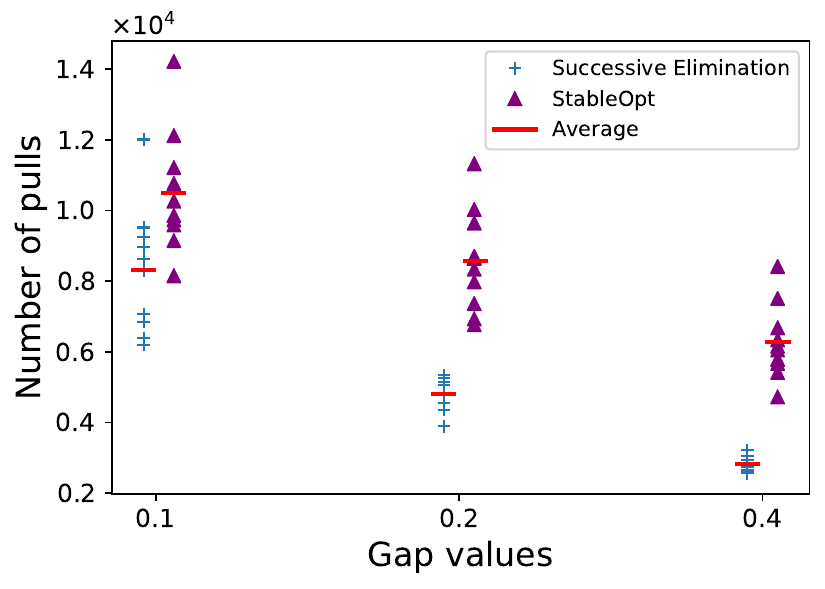}}
    \subfloat[$c=1.5,\eta=0.02$]{
        \includegraphics[width=0.25\textwidth]{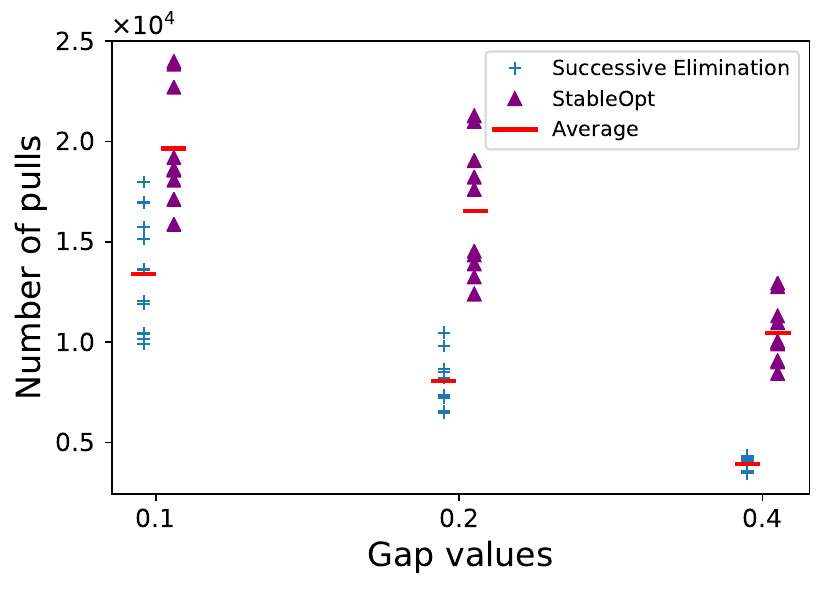}} 
    \subfloat[$c=2,\eta=0.02$]{
        \includegraphics[width=0.245\textwidth]{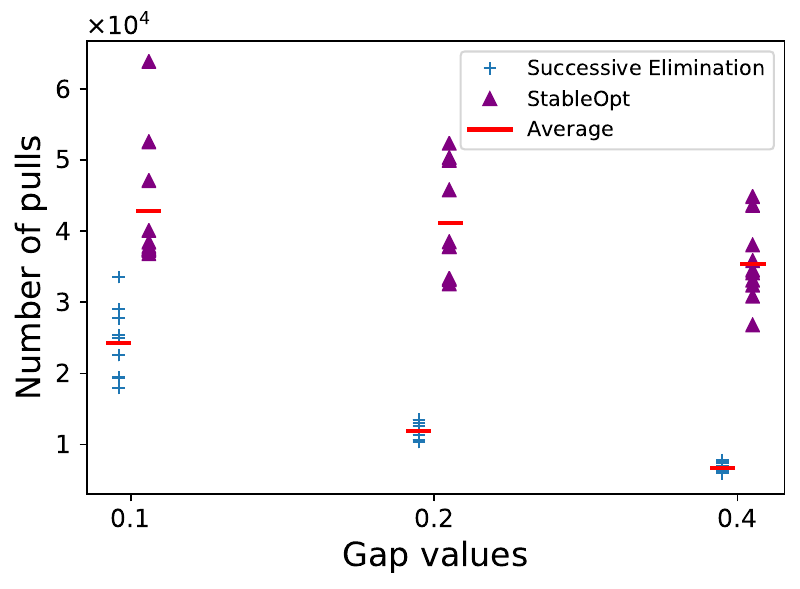}}
    \caption{Comparisons of number of arm pulls for various $c$ (controlling the confidence width), $\eta$ (controlling the {\scshape StableOpt} stopping condition, and $\Delta$ (gap value).)} 
    \label{fig:pull_comp_app}
\end{figure*}

\begin{figure*}[t!]
    \centering
    \subfloat[$c=2,\Delta=0.1$]{
        \includegraphics[width=0.33\textwidth]{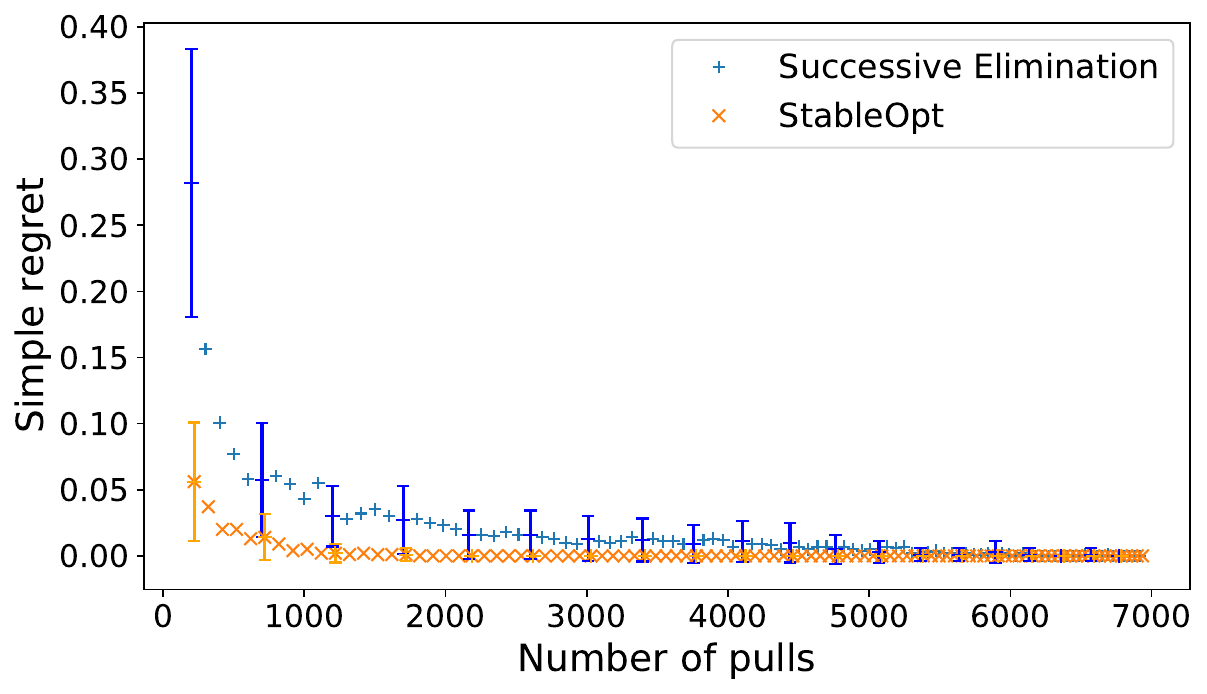}}
    \subfloat[$c=2,\Delta=0.2$]{
        \includegraphics[width=0.33\textwidth]{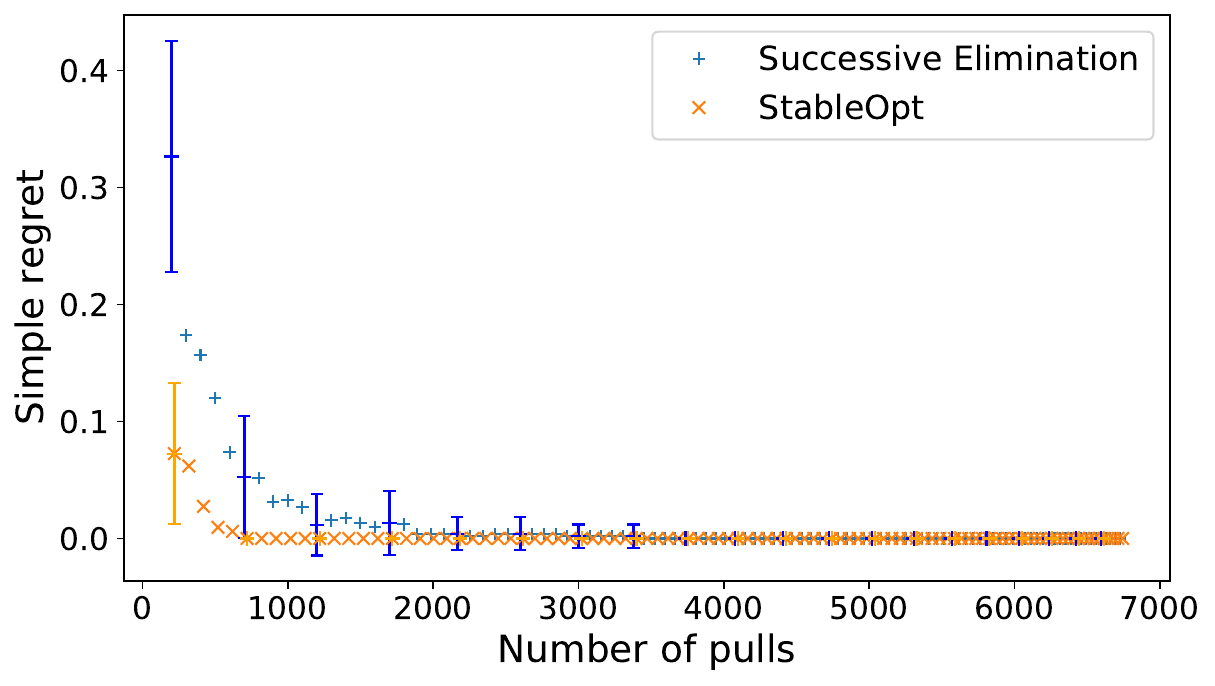}}
    \subfloat[$c=2,\Delta=0.4$]{
        \includegraphics[width=0.33\textwidth]{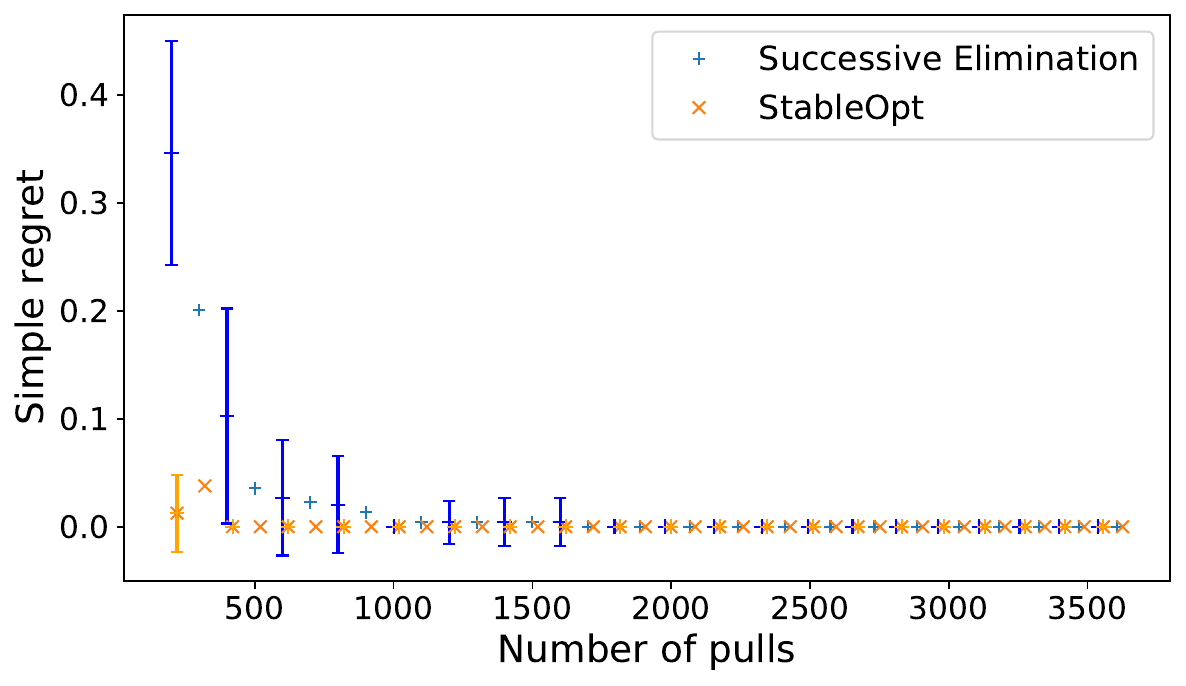}}
    \caption{Simple regret plots with $c=2$ and various gap values.} 
    \label{fig:regret_comp}
\end{figure*}

\begin{table}[t]
    % 	\begin{minipage}{0.49\linewidth}$c=1,\Delta=0.1$
    \centering
    \caption{Empirical success rates for various gap values and $(c,\eta)$-value pairs}
    \label{table:model_acc}
    \resizebox{0.45\textwidth}{!}{
        \begin{tabular}{lccc}
            \hline
            \textbf{Model} & $\Delta = 0.1$ & $\Delta=0.2$ & $\Delta	=0.4$ \\ \hline 
            Elimination ($c=1$) & 1.0 & 1.0 & 1.0 \\
            {\scshape StableOpt} ($c=1, \eta=0.01$) & 0.98 & 0.99 & 1.0 \\
            {\scshape StableOpt} ($c=1, \eta=0.02$) & 0.91 & 1.0 & 1.0 \\
            \hline 
            Elimination ($c=1.25$) & 1.0 & 1.0 & 1.0 \\
            {\scshape StableOpt} ($c=1.25, \eta=0.01$) & 1.0 & 1.0 & 1.0 \\
            {\scshape StableOpt} ($c=1.25, \eta=0.02$) & 1.0 & 1.0 & 1.0 \\
            \hline 
            Elimination ($c=1.5$) & 1.0 & 1.0 & 1.0 \\
            {\scshape StableOpt} ($c=1.5, \eta=0.01$) & 1.0 & 1.0 & 1.0 \\
            {\scshape StableOpt} ($c=1.5, \eta=0.02$) & 1.0 & 1.0 & 1.0 \\
            \hline 
            Elimination ($c=2$) & 1.0 & 1.0 & 1.0 \\
            {\scshape StableOpt} ($c=2, \eta=0.01$) & 1.0 & 1.0 & 1.0 \\
            {\scshape StableOpt} ($c=2, \eta=0.02$) & 1.0 & 1.0 & 1.0 \\
            \hline
        \end{tabular}
    }
\end{table}

Here we explore the effect of varying $c$, the constant in the confidence width $\frac{c}{\sqrt{T_j(t)}}$ (previously set to one), and $\eta$, the confidence width beyond which {\scshape StableOpt} terminates (previously set to $0.01$).

In the top row of \refig{fig:pull_comp_app}, we see that increasing $c$ naturally increases the number of arm pulls for both algorithms (due to more conservative confidence bounds), but appears to impact {\scshape StableOpt} more.  However, the second row indicates that this is at least partly due to the stringent stopping condition, since the less stringent choice $\eta = 0.02$ brings the two algorithms back closer together.  

A caveat here is that increasing $\eta$ puts {\scshape StableOpt} at a higher risk of returning the wrong group; we investigate this in \reftbl{table:model_acc}.  For the most part, the algorithms return the correct group on all 100 trials, but {\scshape StableOpt} indeed starts to produce errors when both $c$ and $\eta$ are chosen too aggressively, particularly $c = 1$ and $\eta = 0.02$.

Finally, in \refig{fig:regret_comp}, we plot the simple regret with $c=2$, in contrast to $c=1$ used in \refig{fig:crit_comp}.  Again, increasing $c$ naturally increases the number of arm pulls for both algorithms, but we observe the same general behavior for both values of $c$.  In general, our findings suggest that {\scshape StableOpt} is highly effective in providing small simple regret, but that more care is needed (compared to Successive Elimination) in choosing the confidence bounds and stopping rule when the goal is exact best-group identification.

\medskip
\noindent {\bf Acknowledgment.} This work was supported by the Singapore National Research Foundation (NRF) under grant number R-252-000-A74-281.

    \bibliography{bibliography}

\end{document}